\documentclass[letterpaper, 10pt, conference]{ieeeconf}
\IEEEoverridecommandlockouts
\usepackage{times}
\newcommand{\citep}{\cite}
\newcommand{\citet}{\cite}
\usepackage{multicol}
\usepackage{xcolor}
\usepackage{colortbl}
\usepackage[bookmarks=true]{hyperref}
\hypersetup{letterpaper,bookmarksopen,bookmarksnumbered,
breaklinks=true,
pdfpagemode=UseOutlines,
colorlinks=true,
linkcolor=blue,
anchorcolor=blue,
citecolor=red,
filecolor=blue,
menucolor=blue,
urlcolor=blue
}
\usepackage{amssymb,amsmath}
\usepackage{gensymb}
\usepackage{cite} %
\usepackage{balance}
\usepackage{booktabs}
\usepackage{caption}
\usepackage[skip=2pt]{caption}
\usepackage[font=footnotesize,labelformat=simple]{subcaption}
\usepackage[font=footnotesize]{caption}
\usepackage{graphicx}
\usepackage{tabularx}
\usepackage{pifont}%

\usepackage{dsfont}

\usepackage{algorithm}
\usepackage{algorithmicx}
\usepackage[noend]{algpseudocode}
\usepackage{amsmath}
\usepackage{xspace}
\usepackage{multirow}
\usepackage{tabularx}
\usepackage{subcaption}
\let\labelindent\relax
\usepackage{enumitem}
\usepackage{tikz}
\usepackage{xcolor}
\usetikzlibrary{fit}
\newcommand\algname[1]{\textsf{#1}\xspace}
\newcommand\astar{\algname{A*}}

\newcommand\hastar{\algname{HA*}}
\newcommand\ighastar{\algname{IGHA*}}
\newcommand\biighastar{\algname{Bi-IGHA*}}
\newcommand\alg{\algname{ALG}}

\newcommand\sst{\algname{SST}}
\newcommand\idba{\algname{idB-A*}}

\newcommand\activate{\textsc{Activate}\xspace}
\newcommand\shift{\textsc{Shift}\xspace}
\newcommand\nearmeet{\textsc{NearMeet}\xspace}
\newcommand\gnmv{\textsc{GetNearMeetVertex}\xspace}
\newcommand\forward{\textsc{Forward}\xspace}
\newcommand\backward{\textsc{Backward}\xspace}

\algdef{SE}[DOWHILE]{Do}{DoWhile}{\algorithmicdo}[1]{\algorithmicwhile\ #1}%
\newtheorem{definition}{Definition}[section]

\newtheorem{theorem}{Theorem}[section]

\newtheorem{corollary}{Corollary}[section]

\newtheorem{observation}{Observation}[section]
\DeclareMathOperator{\Succ}{\textsc{Succ}}
\DeclareMathOperator*{\argmin}{arg\,min}

\DeclareRobustCommand{\legendsolid}[1]{%
  \textcolor[rgb]{#1}{\raisebox{0.3ex}{\rule{0.5cm}{2pt}}}%
}
\graphicspath{{./}}

\begin{document}
\title{Bidirectional Incremental Generalized Hybrid A*}
\author{\authorblockN{Sidharth Talia, Oren Salzman, and Siddhartha Srinivasa}}
\maketitle
\begin{abstract}
We focus on the problem of efficient anytime kinodynamic planning for systems with complex dynamics in unstructured environments that make using precomputed motion primitives infeasible.
Directly applying \astar here is computationally infeasible due to the curse of dimensionality.
Methods such as Hybrid \astar~(\hastar) address this by pruning the search tree by discretizing the state space, but the coupling between pruning and discretization resolution can eliminate a vertex on the solution path.
The Incremental Generalized Hybrid \astar~(\ighastar) breaks this coupling by organizing anytime search over a hierarchy of resolutions and by \textit{freezing} vertices for later expansion rather than pruning. 
However, \ighastar can still freeze a vertex on the solution path, forcing the search to spend expansions elsewhere before reaching it. 
Our \textbf{\textit{key insight}} is that bidirectional-\ighastar~(\biighastar) not only gains the expected reduction in effective search depth from classical bidirectionality, but also mitigates this \textit{frozen-vertex barrier}: when one search is blocked, the opposing search can reach it at a low resolution.
We formalize this structural separation between \biighastar and \ighastar
and empirically show a reduction in effective branching factor beyond the theoretical expectation from classic bidirectionality alone.
Through open-loop experiments in $\mathbb{R}^3$, $\mathbb{R}^4$, and $\mathbb{R}^6$, we show that \biighastar substantially reduces expansions over \ighastar. Closed-loop experiments further demonstrate improved performance and feasibility in simulation and on a real-world robotic system.
Link to \href{https://personalrobotics.github.io/IGHAStar/biighastar.html}{Website}.
\end{abstract}
\vspace{-10pt}
\section{Introduction}

Planning for autonomous systems in complex, unstructured environments—such as high-speed off-road autonomy~\citep{han2023model, talia2025incremental} and aerial robotics~\citep{uav_path_planning}—is challenging due to complex, nonlinear dynamics.
We are motivated by kinodynamic planning for high-speed off-road autonomous vehicles, where precomputing motion trees is often infeasible because the terrain, friction, and obstacle geometry continuously evolve~\citep{han2023model, meng2023}.
In many such settings, the planner has access only to a computationally expensive forward dynamics model rather than a fast analytical solution to the two-point boundary value problem (2pBVP).

For such problems, directly applying \astar\ search in continuous kinodynamic spaces is computationally prohibitive due to the curse of dimensionality.
Hybrid \astar\ (\hastar)~\citep{dolgov2010path} addressed this challenge by using \textit{approximate dominance}~\citep{bellman_curse}; discretizing the state space into grid cells at some \textbf{\textit{resolution}},
with each cell only retaining the lowest-cost vertex to reduce the effective branching factor~(Sec.~\ref{sec: surprise}).
The underlying assumption is that the \textit{best cost} from any vertex in the grid cell to the goal is identical.
However, this assumption may not always hold, and \hastar can even prune a goal-bearing vertex~(a vertex on the path to the goal) from the search tree.
Incremental Generalized Hybrid \astar\ (\ighastar)~\citep{talia2025incremental} avoids this by organizing search over a hierarchy of resolutions in an anytime fashion and \textbf{\textit{freezing}} vertices for expansion in later search iterations rather than pruning them.
However, at a particular resolution, \ighastar can still freeze a goal-bearing vertex, creating a barrier the search overcomes only through many more expansions.
We illustrate this through Fig.~\ref{fig:main} and its caption, and term this the \textbf{\textit{frozen vertex barrier}}.

Our \textbf{\textit{key insight}} is that when we run two \ighastar searches, one \textit{forward} and one \textit{backward}, yielding solutions when the opposing trees \textbf{\textit{near meet}}~\citep{lavalle2001randomized}~(Def.~\ref{def:nearmeet}), then, in addition to the expected reduction in effective branching factor from classical bidirectionality~\citep{HolteFSSC17}, we also get an additional reduction from \textbf{\textit{bidirectional mitigation}} of the frozen vertex barrier.
Even if one search gets blocked by frozen vertices, the other search can still meet it at a low resolution~(Fig.~\ref{fig:main_explanation}).

In this work, we identify the frozen vertex barrier as a consequence of approximate dominance-based search~(Sec.~\ref{sec:igha_overview},~Obs.~\ref{obs:frozen_barrier}),
propose \textit{bidirectional}-\ighastar~(\biighastar,~Sec.~\ref{sec:bi-igha}) that preserves \ighastar's guarantees on monotonic cost improvement and termination with finite expansions~(Thm.~\ref{thm:igha_equal}), and formally show the structural separation between \ighastar and \biighastar as a result of the bidirectional mitigation~(Sec.~\ref{subsec: freezing_theory}).
We empirically show that \biighastar provides $\approx 10 \times$ reduction in expansions on planning problems in $\mathbb{R}^3$, $\mathbb{R}^4$, and $\mathbb{R}^6$~(Fig.~\ref{fig:Success_rate_cost_comparison_open_loop}) and validate that claim regarding the additional reduction in effective branching factor~(Fig.~\ref{fig:b*_ratio}).
Using closed-loop planning and control experiments in simulation~($\mathbb{R}^4$) as well as in the real world high speed off-road autonomy, we show that \biighastar significantly outperforms \ighastar.

\begin{figure*}[t]
\centering

\begin{minipage}{0.66\linewidth}
\centering

\begin{subfigure}{0.23\linewidth}
  \includegraphics[width=\linewidth]{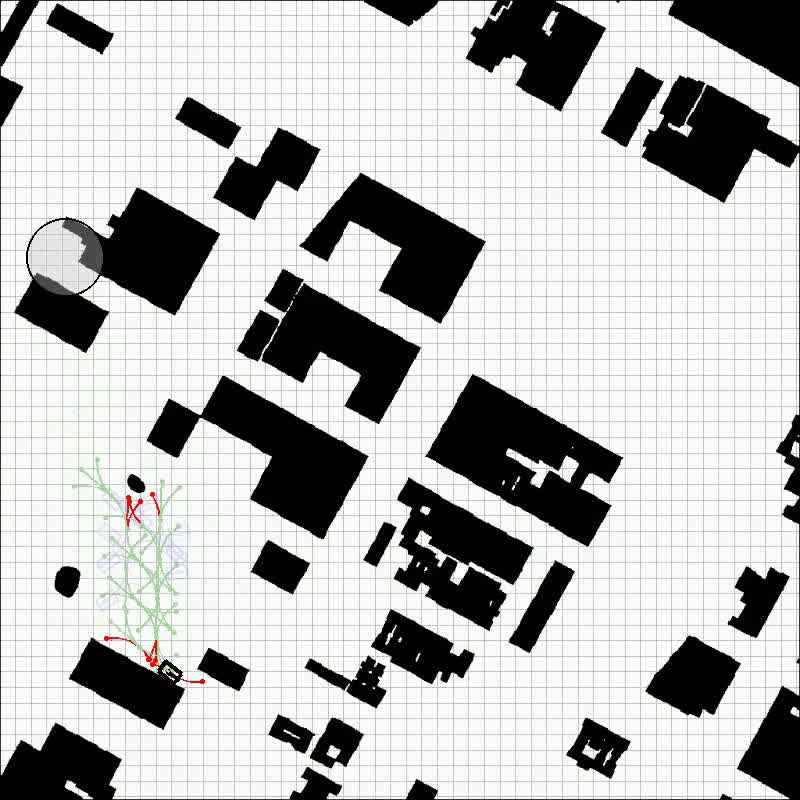}
  \caption{\ighastar starts}
  \label{fig:igha_anecdote_3DoF_0}
\end{subfigure}
\hspace{1pt}
\begin{subfigure}{0.23\linewidth}
  \includegraphics[width=\linewidth]{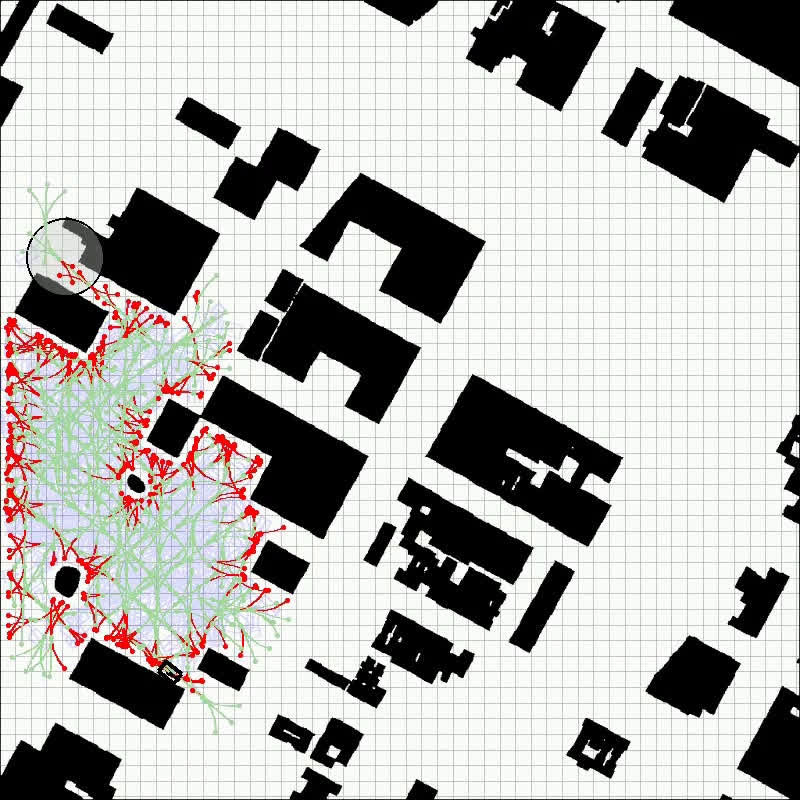}
  \caption{Missed Goal}
  \label{fig:igha_anecdote_3DoF_1}
\end{subfigure}
\hspace{1pt}
\begin{subfigure}{0.23\linewidth}
  \includegraphics[width=\linewidth]{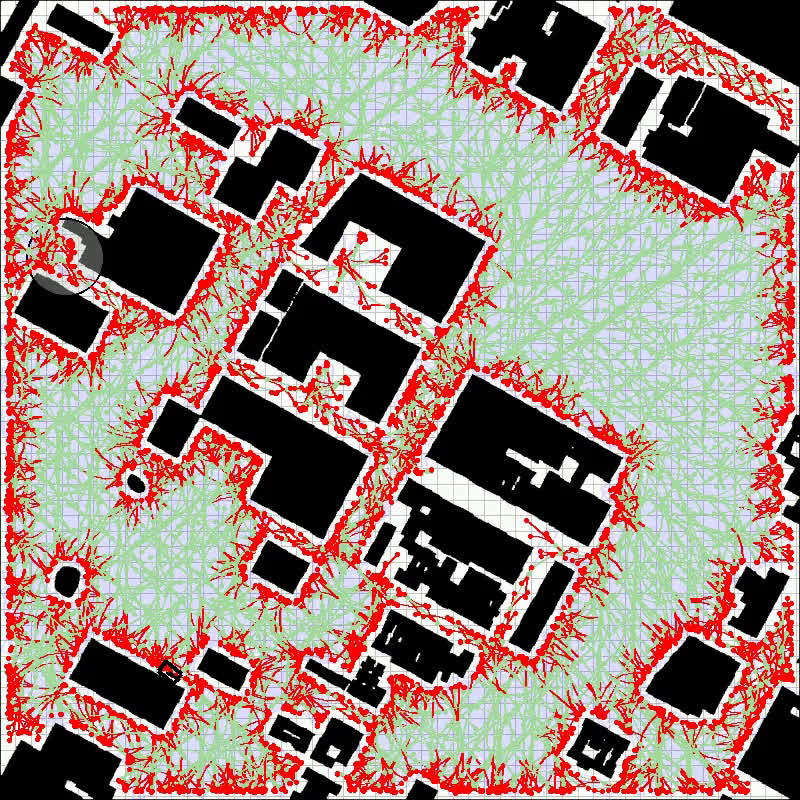}
  \caption{Waste expansions}
  \label{fig:igha_anecdote_3DoF_2}
\end{subfigure}
\hspace{1pt}
\begin{subfigure}{0.23\linewidth}
  \includegraphics[width=\linewidth]{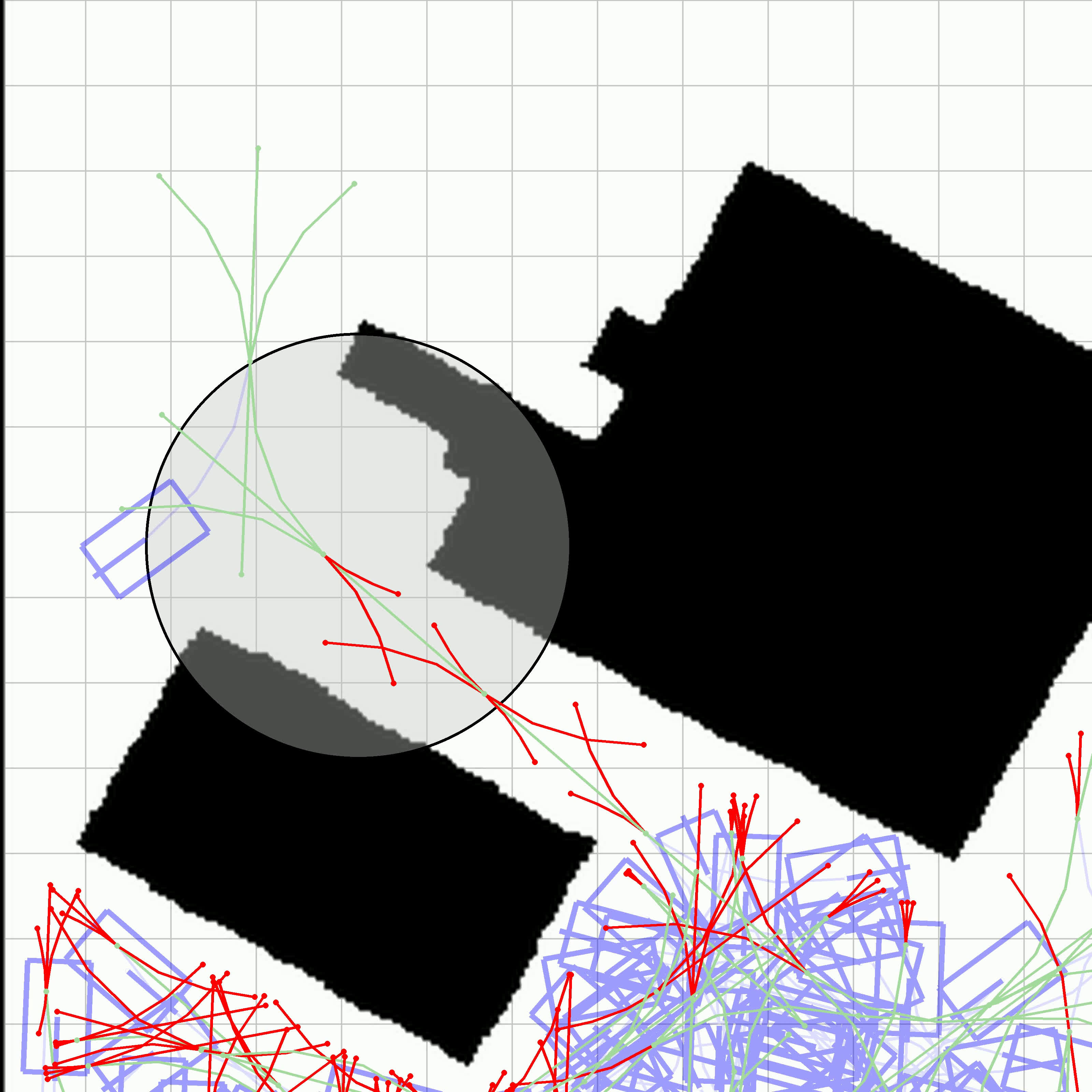}
  \caption{What happened?}
  \label{fig:igha_anecdote_3DoF_3}
\end{subfigure}

\begin{subfigure}{0.23\linewidth}
  \includegraphics[width=\linewidth]{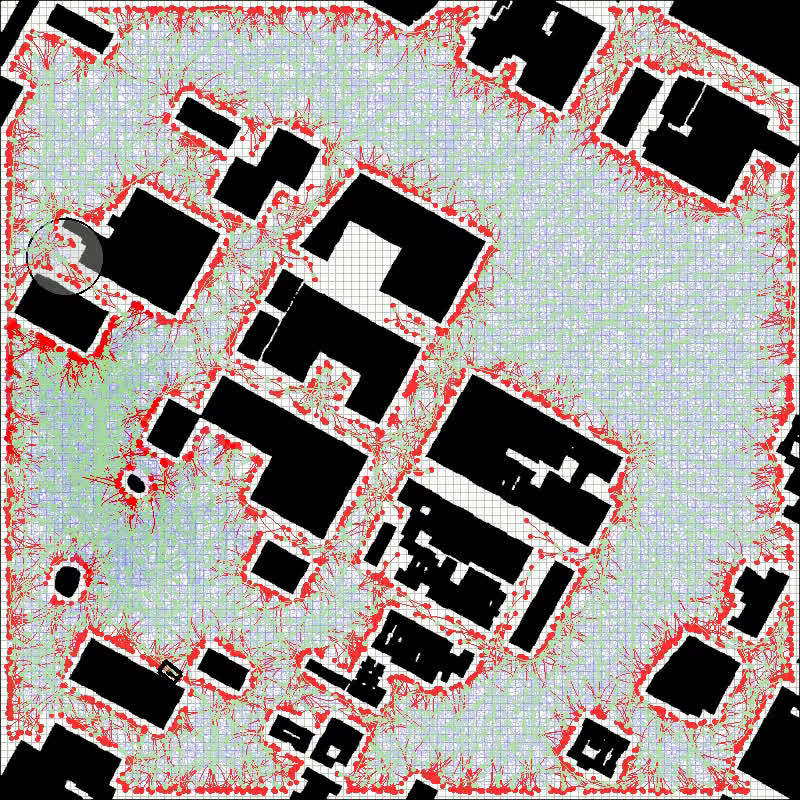}
  \caption{Increase resolution}
  \label{fig:igha_anecdote_3DoF_4}
\end{subfigure}
\hspace{1pt}
\begin{subfigure}{0.23\linewidth}
  \includegraphics[width=\linewidth]{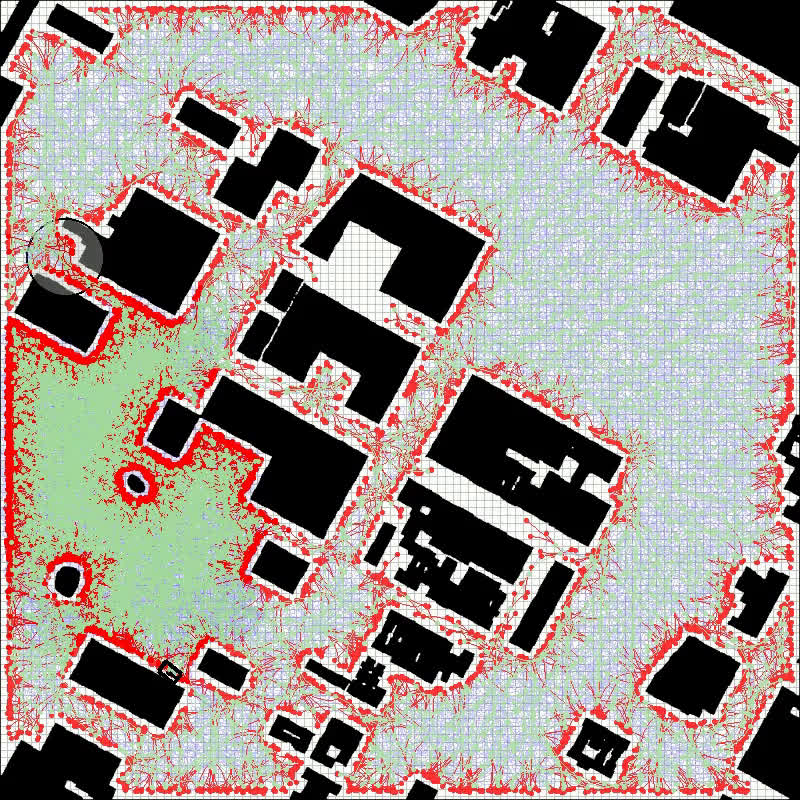}
  \caption{Lots of expansions}
  \label{fig:igha_anecdote_3DoF_5}
\end{subfigure}
\hspace{1pt}
\begin{subfigure}{0.23\linewidth}
  \includegraphics[width=\linewidth]{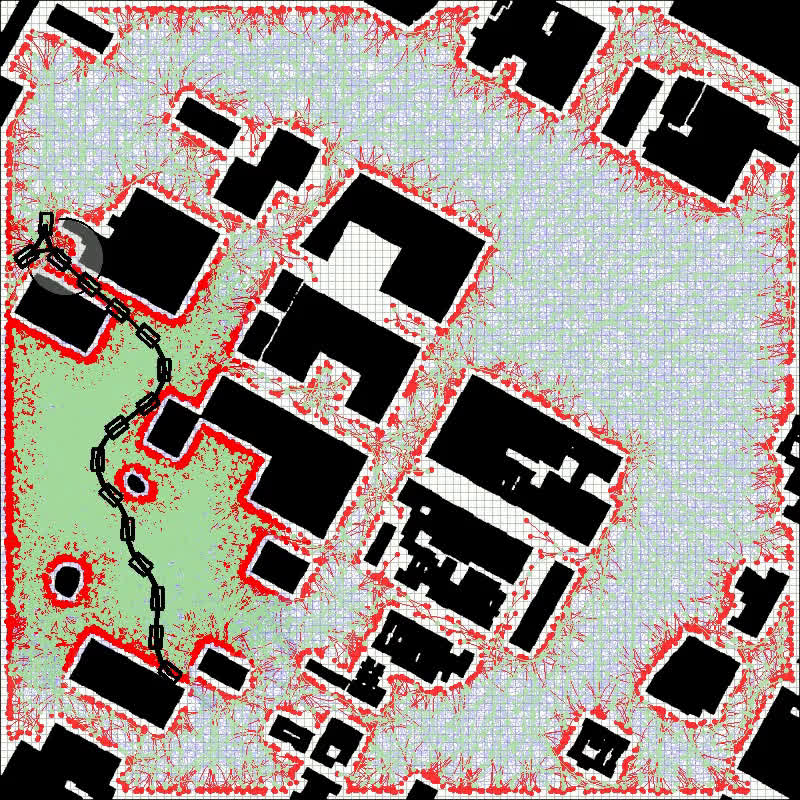}
  \caption{Encounter path}
  \label{fig:igha_anecdote_3DoF_6}
\end{subfigure}
\hspace{1pt}
\begin{subfigure}{0.23\linewidth}
  \includegraphics[width=\linewidth]{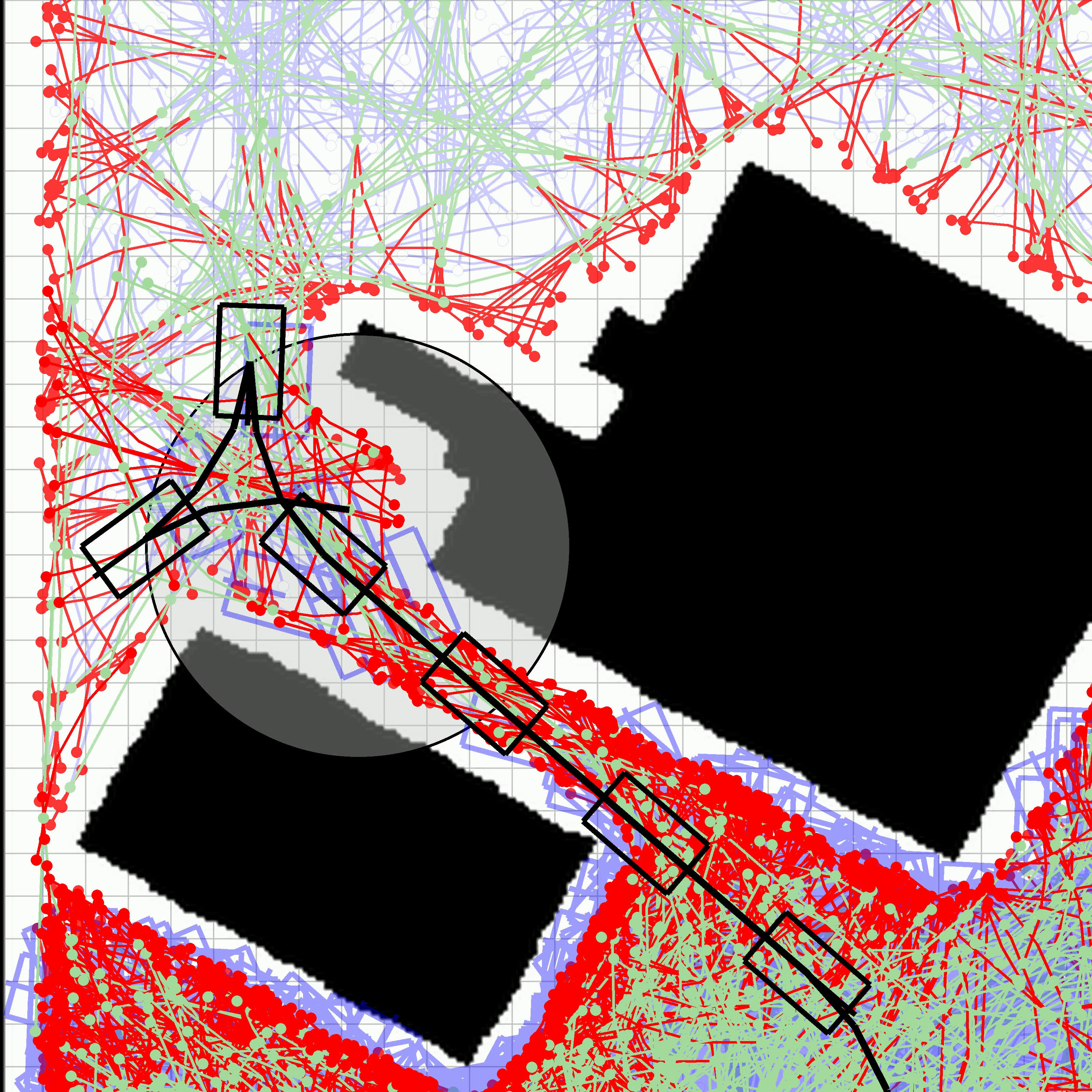}
  \caption{Misdirect=$\uparrow$Effort}
  \label{fig:igha_anecdote_3DoF_7}
\end{subfigure}
\end{minipage}
\begin{minipage}{0.33\linewidth}
\centering
\begin{subfigure}{\linewidth}
  \includegraphics[width=\linewidth]{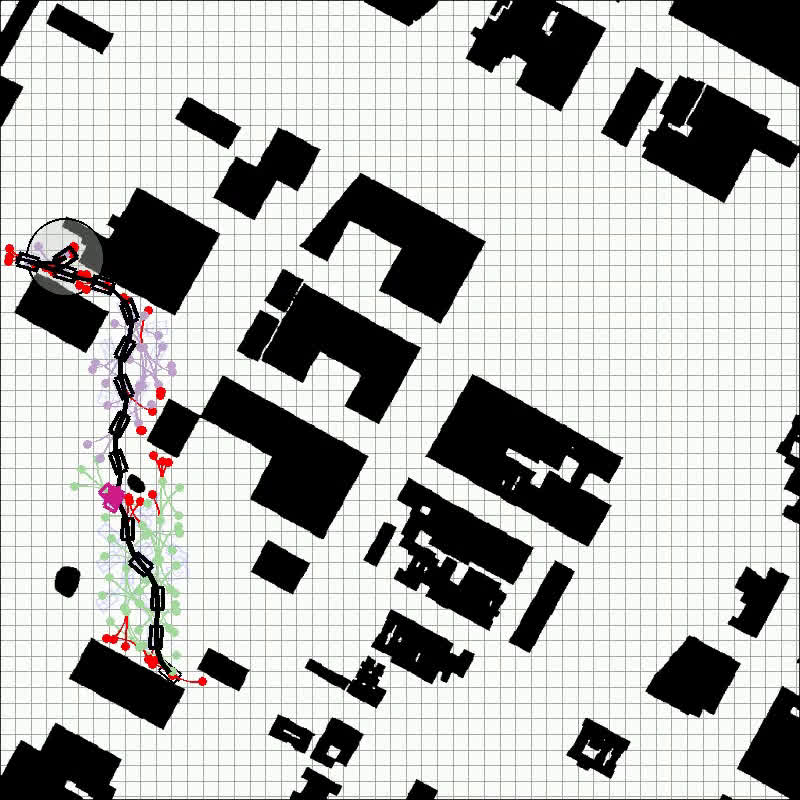}
  \caption{Why does \biighastar not struggle?}
  \label{fig:biigha_anecdote_3DoF_0}
\end{subfigure}
\end{minipage}

\caption{
Comparison of \ighastar~\subref{fig:igha_anecdote_3DoF_0}–\subref{fig:igha_anecdote_3DoF_7} and \biighastar~\subref{fig:biigha_anecdote_3DoF_0}.
For \ighastar, forward expansions are \legendsolid{0.651,0.859,0.627}, frozen (non-dominant) vertices \legendsolid{0.4,0.4,0.9}, and collision states \legendsolid{1.000,0.750,0.750}.
For \biighastar, \forward search uses the same color, \backward search is \legendsolid{0.761,0.647,0.812}, lighter shades denote frozen vertices, and the \textit{near-meet} is \legendsolid{1.000,0.0,0.750}.
In both, the final path is \legendsolid{0.000,0.000,0.000}.
The goal condition is the Reeds-Shepp~\citep{reeds1990optimal} distance to $v_g \in \mathbb{R}^3$
Although \ighastar appears to pass through the goal region in~\subref{fig:igha_anecdote_3DoF_1}, the true goal configuration is in $(x,y,\theta)$, and the goal condition requires the Reeds-Shepp distance to be within a threshold~(the circle is a visual aid), it freezes the critical vertex leading to $v_g$~(\subref{fig:igha_anecdote_3DoF_3}), forcing exploration of the remaining free space~(\subref{fig:igha_anecdote_3DoF_2}), resolution refinement after exhaustion~(\subref{fig:igha_anecdote_3DoF_4},~\ref{fig:igha_anecdote_3DoF_5}), and eventual discovery only after many expansions~(\subref{fig:igha_anecdote_3DoF_6}).
\biighastar avoids this failure mode~(\subref{fig:biigha_anecdote_3DoF_0}) by mitigating the frozen vertex barrier.
}
\label{fig:main}
\vspace{-20pt}
\end{figure*}

\section{Related Work}

\subsection{Bidirectional Graph Search}
Bidirectional graph search~(see, e.g.,~\cite{SiagSFS23,SahamFCS17,ShahamFSR18,ShperbergFSSH19,ShperbergFSSH19b,SturtevantSFC20,alcazar2020unifying}) algorithms have long been employed to improve search efficiency.
Recent theoretical advances in this domain have primarily followed two distinct conceptual frameworks \cite{sturtevant2018brief,wang2025meet}. 
The first, grounded in the theory of Must-Expand Pairs~(\textbf{MEP})~\cite{eckerle2017sufficient}, led to algorithms like Near-Optimal Bidirectional Search~(NBS) \cite{chen2017nbs}. 
NBS introduced mechanisms to certify optimality even when the two search frontiers do not meet exactly at a vertex.
The second framework focuses on the Meet-in-the-Middle Property~(\textbf{MMP}), which guarantees that no state is expanded if its cost exceeds half the optimal solution cost~\cite{HolteFSSC17},
with recent work~\citep{wang2025meet} showing ways to reduce computational burden of computing termination conditions~\cite{sadhukhan2012new,sewell2021dynamically}.
However, both MEP and MMP approaches typically assume a static graph in which vertices can be checked for exact overlap. 
This assumption often breaks down in kinodynamic planning for off-road autonomy

\subsection{Kinodynamic Bidirectional Motion-Planning}\label{sec: kinodynamic_bi_planning}
In kinodynamic planning,
unlike geometric graph search, the forward and backward trees in continuous state space
do not share exact states.
Consequently, establishing a connection requires solving a local BVP to bridge the gap between the two trees \cite{lavalle2006planning}. 
Early work on Kinodynamic Bi-RRT utilized the property of a \textit{local controllability radius}~(\textbf{LCR}) to assume that if two states are sufficiently close, they can be connected \cite{lavalle2001randomized}.
Similarly, recent sampling-based methods verify a connection between vertices using trajectory optimization~\citep{choudhury2016regionally}, or exploit the reverse tree only when vertices fall within a specific radius \cite{nayak2022bidirectional}.
Approaches such as \idba~\citep{ortiz2024idba, ortiz2024idbrrt} rely on the concept of bounded discontinuity for connecting motion primitives.
At their core, they all rely on the ability to connect vertices within some LCR, either by assuming that a connection is possible or at least to filter out unpromising candidates.
Our algorithm also uses the notion of LCR, but remains orthogonal to how the connection is actually verified.

\subsection{Kinodynamic Planning \& Off Road Autonomy}
For unstructured and ever-changing environments with complex dynamics, such as in the case of off-road autonomy~\cite{han2023model, terrainCNN}, caching motion primitive rollouts~\citep{ortiz2024idba, ortiz2024idbrrt}
is often infeasible \cite{talia2025incremental} as the dynamics depend on the terrain geometry, and would affect the rollout as soon as the perception system provides a new map.
Algorithms like Hybrid $A^*$~(\hastar) \cite{dolgov2010path} and \sst \cite{SST} commit to tree search with the system's dynamics model, and leverage the notion of \textit{approximate dominance}~\citep{bellman_curse} to reduce the computational complexity of these problems; however, these approaches couple the tree discovery with the discretization resolution.
While the \ighastar~\citep{talia2025incremental} breaks this coupling outside the inner search process, it is still susceptible to the frozen vertex barrier.
In this work, we extend the \ighastar framework into the bidirectional setting and show how it mitigates this issue.
\vspace{-15pt}

\setlength{\textfloatsep}{0pt}
\begin{figure*}[t]
\centering
\begin{subfigure}{0.16\linewidth}
  \centering
  \includegraphics[width=\linewidth]{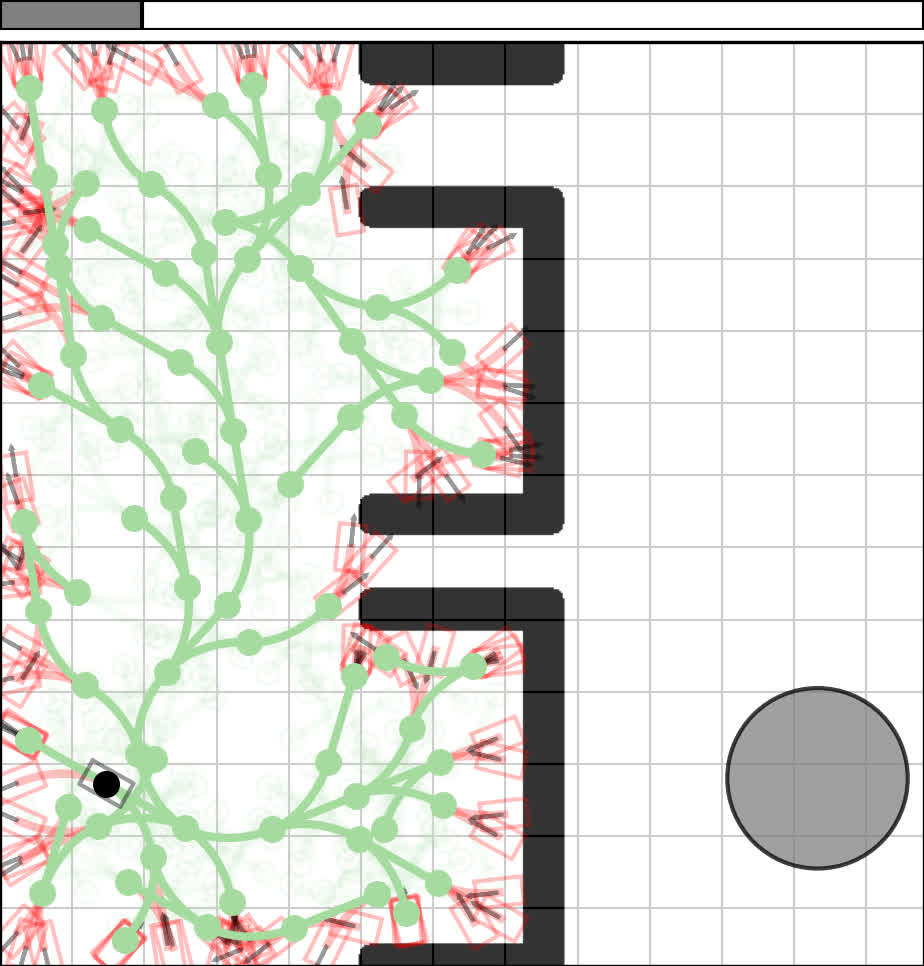}
  \caption{1$\times$Res., Uni}
  \label{fig:main_0}
\end{subfigure}
\begin{subfigure}{0.16\linewidth}
  \centering
  \includegraphics[width=\linewidth]{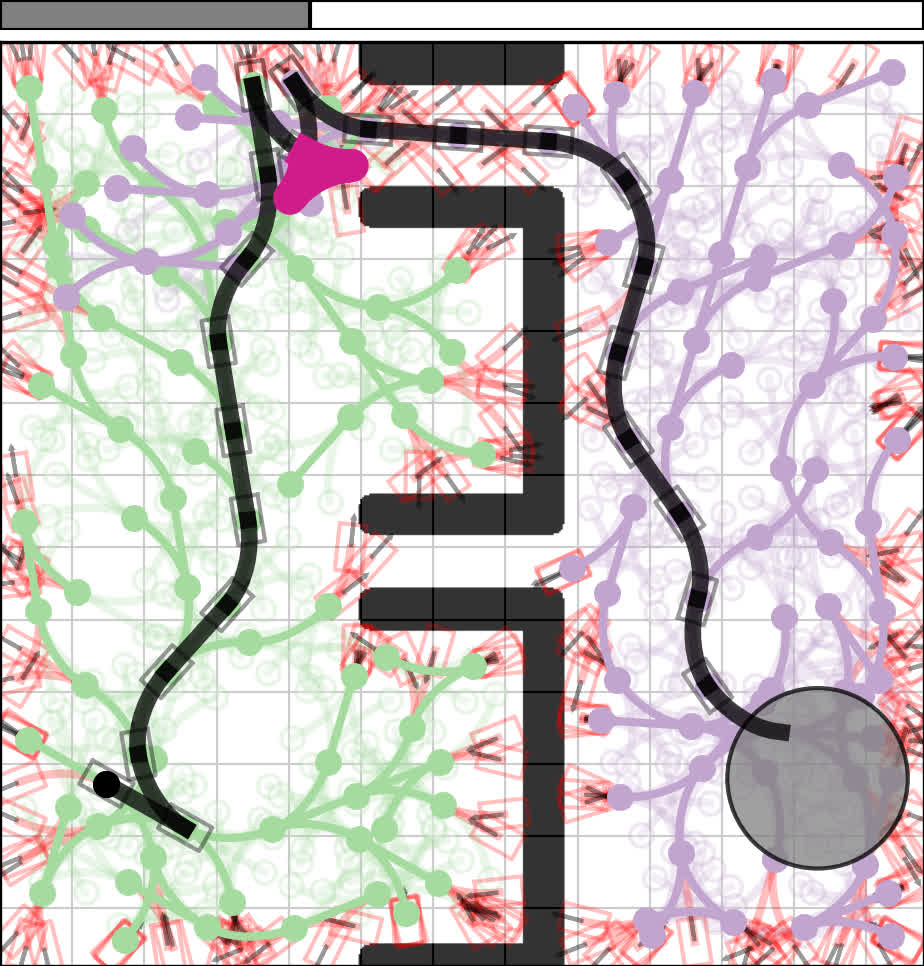}
  \caption{1$\times$Res., Bi}
  \label{fig:main_1}
\end{subfigure}
\begin{subfigure}{0.16\linewidth}
  \centering
  \includegraphics[width=\linewidth]{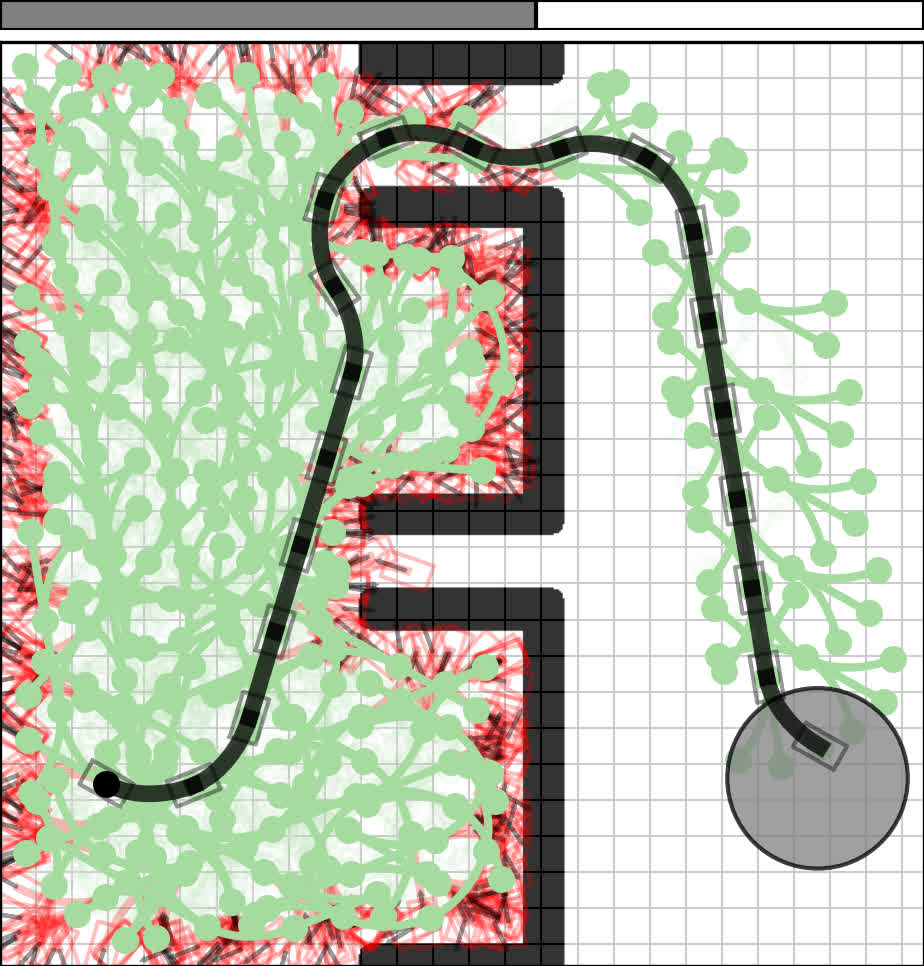}
  \caption{4$\times$Res., Uni}
  \label{fig:main_2}
\end{subfigure}
\begin{subfigure}{0.16\linewidth}
  \centering
  \includegraphics[width=\linewidth]{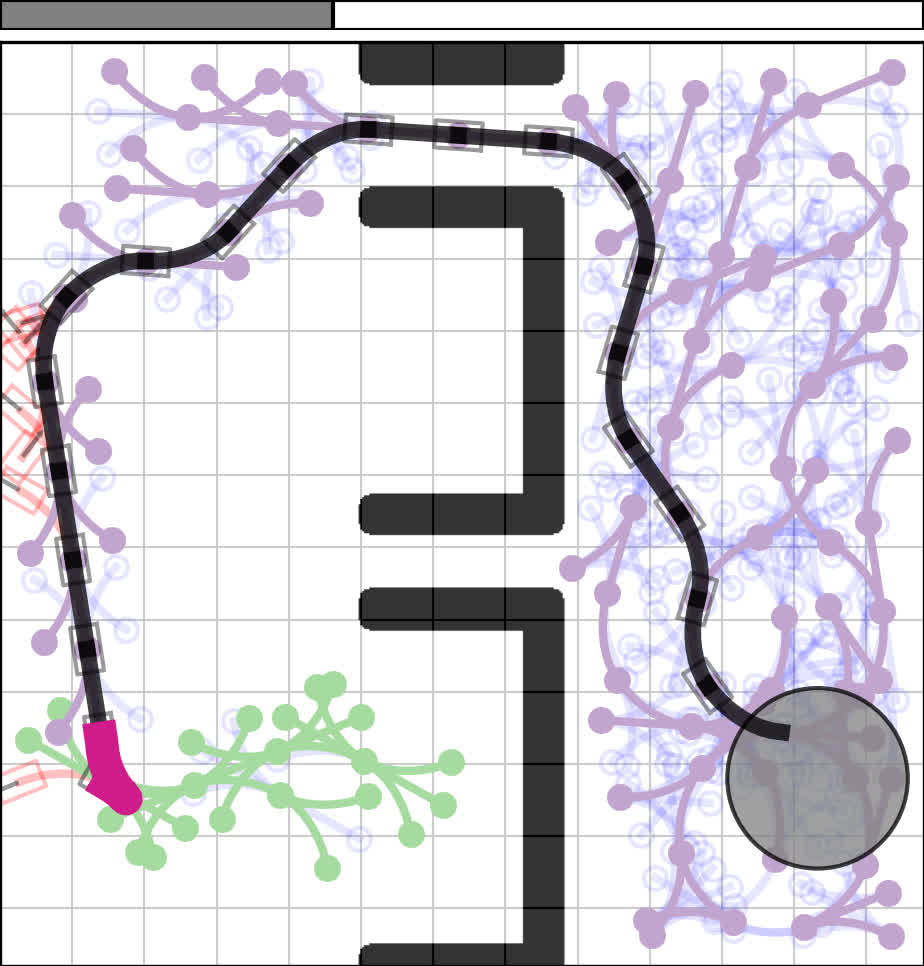}
  \caption{1$\times$Res., Bi}
  \label{fig:main_3}
\end{subfigure}
\begin{subfigure}{0.16\linewidth}
  \centering
  \includegraphics[width=\linewidth]{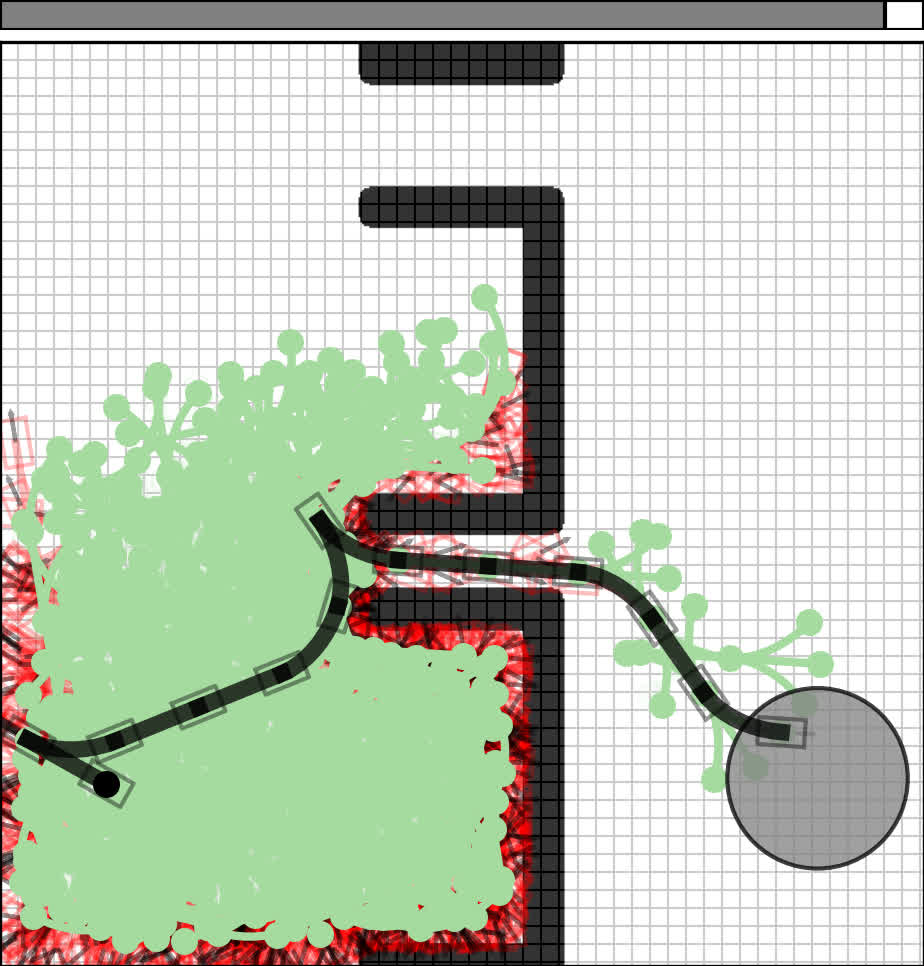}
  \caption{16$\times$Res., Uni}
  \label{fig:main_4}
\end{subfigure}
\begin{subfigure}{0.16\linewidth}
  \centering
  \includegraphics[width=\linewidth]{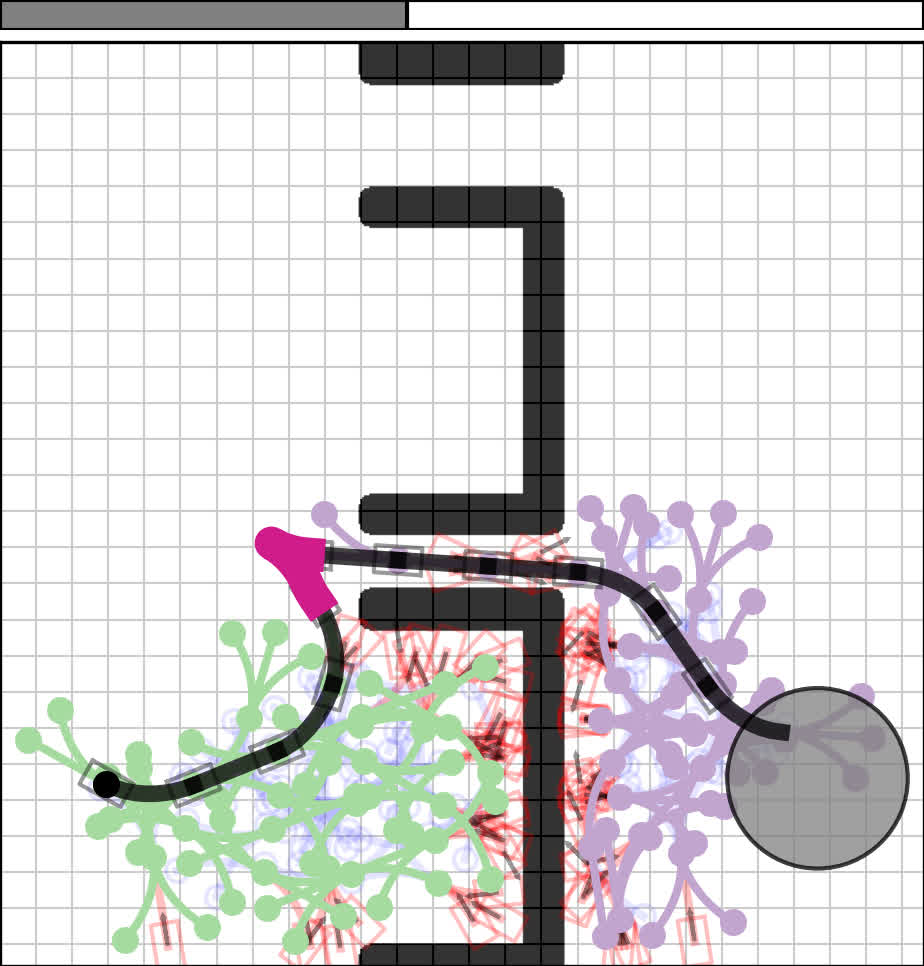}
  \caption{4$\times$Res., Bi}
  \label{fig:main_5}
\end{subfigure}
\caption{
Comparison of \ighastar and \biighastar illustrating \biighastar attaining equivalent solutions at lower resolutions, manifesting bidirectional mitigation.
The color scheme matches Fig.~\ref{fig:main}; the grey bar indicates total expansions at each snapshot. 
For 
clarity, we only render the vertices expanded at an iteration, and use an \activate variant that resets $Q^D_v$, which removes the advantage of vertex reuse.
At the base resolution, \ighastar fails~(\subref{fig:main_0}), whereas the \backward search in \biighastar succeeds~(\subref{fig:main_1}). Unidirectional \ighastar must refine the resolution to obtain first a suboptimal~(\subref{fig:main_2}) and then an optimal solution~(\subref{fig:main_4}). In contrast, \biighastar recovers equivalent-cost solutions at lower resolutions via near-meets~(\legendsolid{1.000,0.0,0.750})~(\subref{fig:main_3},~\subref{fig:main_5}), using much fewer expansions.
}
\label{fig:main_explanation}
\vspace{-20pt}
\end{figure*}
\section{Problem Statement}\label{sec:problem_statement}
We consider the problem of computing near-optimal paths over two implicitly-defined opposing trees $T_D=(V_D,E_D)$ for $D\in \{F,B\}$~(\forward,~\backward resp.), $V_D\subseteq \mathbb{R}^N$.
For the $D$ search, each edge $(u, v) \in E_D$ has an associated positive cost given by a function $w : E_D \rightarrow \mathbb{R}^+$.
$T_D$ is rooted at a (given) root vertex $v^D_r$, and the objective is to reach any vertex in a target set $V^D_T$ while minimizing the accumulated edge cost, where $V^D_T$ denotes the subset of $V_D$ satisfying the goal condition.
The tree structures are defined implicitly by a successor function $\Succ$ that returns the set of valid children for any given vertex.
With slight abuse of notation, we set $w(\pi)$ to be the cost of a path $\pi = \langle v_1, v_2, \ldots, v_n \rangle$. 
Namely,
$w(\pi):=\sum_{i=1}^{n-1} w(v_i, v_{i+1})$ and define
$w(\emptyset) = \infty$ to denote the cost of failure.
We assume access to a heuristic function $h_D: V_D \rightarrow \mathbb{R}_{\ge 0}$ that is admissible, providing an estimate of the remaining cost to the goal for the $D$ search.
In \forward search, $v^F_r \gets v_s\in V_s \subseteq V_F$ and the target set $V^F_T \gets V_g$, and in \backward search, $v^B_r \gets v_g\in V_g \subseteq V_B$ and $V^B_T \gets V_s$, where $s,~g$ refer to start and goal resp.
We assume that every \backward edge corresponds to a \forward realizable edge of equal cost, so reversing a \backward edge preserves feasibility.
Notationally, let $\bar{D}$ imply the opposing direction; i.e., $T_{\bar{D}}$ is $T_D$'s opposing tree.

\section{Algorithm background--\ighastar}\label{sec:igha_overview}
Recall that \astar~\citep{a_star} on an implicitly defined tree $T_D=(V_D,E_D)$ maintains a priority queue $Q^D_v$ ordered by $f$-value, where $v.f=v.g+h_D(v)$, $v.g$ is the cost-to-come, and $h_D(v)$ is the heuristic cost-to-go. It repeatedly expands vertices from $Q^D_v$ using $\Succ$ until a goal-satisfying vertex is expanded. This is prohibitively expensive in the continuous state space setting as exact state overlap is measure zero.
Hybrid \astar~(\hastar~\citep{dolgov2010path}) reduces this burden through \textbf{\textit{approximate dominance}}. 
The state space is discretized at resolution $R$, and $\hat{v}:V_D\rightarrow V_D$ maps each vertex $v$ to the vertex currently dominating its region (implemented as a grid cell). 
When a new vertex enters a dominated region, the lower-$g$ vertex is retained and the other pruned. 
This assumes the \textit{best cost} from any vertex in a cell to the goal is identical, coupling pruning to $R$: a fine $R$ approaches the computational burden of \astar, while a coarse $R$ can prune the goal-bearing vertex.

\ighastar (Alg.~\ref{alg:bi-igha}; left half of Fig.~\ref{fig:flowchart}, excluding \textcolor[HTML]{0AA844}{green} portions) breaks this coupling by generalizing approximate dominance to $\mathbf{N}$ \textbf{\textit{levels}} $\mathbf{\ell}=[\ell_0,\ell_1,\dots,\ell_\mathbf{N}]$ with corresponding \textbf{\textit{resolutions}} $\mathbf{R}=[R_0,R_1,\dots,R_\mathbf{N}]$ (we use ``level'' and ``resolution'' interchangeably), with $\hat{v}:V_D\times\mathbb{N}\rightarrow V_D$ indexing the level $l_D$. 
Unlike \hastar, \ighastar retains non-dominant vertices but marks them \textbf{\textit{InActive}} or \textbf{\textit{Frozen}}, preventing their expansion during the current search iteration.
The search starts at level 0 with root $v^D_r$. 
The planner {\activate}s $v^D_r$ and begins an inner \textbf{\textit{search iteration}} (Line~\ref{ll:igha_search_loop_start}), expanding vertices from $Q^D_v$ (Line~\ref{ll:igha_popactive}) until the goal condition is satisfied, $Q^D_v.\text{Active}=\emptyset$, or a resolution \shift occurs (Lines~\ref{ll:igha_search_loop_start},~\ref{ll:igha_peek_active},~\ref{ll:igha_goal_check},~\ref{ll:shift_call}).
Upon termination, \textsc{Bound} (Line~\ref{ll:igha_bound}) removes vertices worse than the best cost so far, $w(\hat{\pi})$; the remaining vertices are \textsc{Project}ed to the new resolution (Line~\ref{ll:igha_project}), and the inner search is bootstrapped with newly {\activate}d vertices (Line~\ref{ll:igha_activate}). 
This repeats until $Q^D_v=\emptyset$.

Although \ighastar avoids \hastar's pruning of goal-bearing vertices, it can still freeze one at a given resolution, forcing many expansions of vertices that do not lead to the goal before refinement; the goal-bearing vertex may remain frozen even after refinement. This creates a barrier that the search overcomes only through additional expansions (Fig.~\ref{fig:main}), which we term the \textbf{\textit{frozen vertex barrier}}.
Here we consider running two \ighastar searches to mitigate this, one \forward and one \backward, so that either search can make progress toward the other even when one becomes blocked by frozen vertices.

\section{Bidirectional \ighastar\ (Bi-\ighastar)}\label{sec:bi-igha}

\subsection{Considerations for Making \ighastar Bidirectional}
While bidirectional search offers the potential of exponential speed-ups~\citep{wang2025meet, HolteFSSC17} through the reduction of the search depth,
these techniques typically assume a static graph, where vertices overlap exactly, and that the dynamics are symmetric forward and backward~\citep{kwa1989bs}.
Once a vertex is reached, it is assumed that calling $\Succ$ would produce the same neighbours from both sides, and so we no longer have to continue the search \textit{through} that vertex.
The requirement of exact overlap is important for efficiency in finding \textit{a path}, and the requirement of producing the same neighbours on both sides is important for reducing search effort \textit{after} finding the best path (early termination).
We leverage the notion of a \textit{local controllability radius}, standard in kinodynamic planning~\citep{lavalle2006planning}, for detecting and verifying \textbf{\emph{near meets}} to achieve efficiency in finding \textit{a path}.
\begin{definition}[Local Controllability Radius]
\label{def:lcr}
A system has local controllability radius $R_\text{LCR}$ if for any states $v\in V_D, \bar{v}\in V_{\bar{D}}$ with $\|v-\bar{v}\|_2 < R_\text{LCR}$, there exists a dynamically feasible trajectory connecting $v$ to $\bar{v}$.
\end{definition}
\begin{definition}[\nearmeet]\label{def:nearmeet}
for a $v \in V_D$, \nearmeet
returns \textit{true} if
$
\exists \bar{v} \in V_{\bar{D}}
\text{ such that }
\|v-\bar{v}\|_2 < R_\text{LCR}
$
and the dynamically feasible trajectory
connecting $v$ to $\bar{v}$ is collision free.
It returns \textit{false} otherwise.
\end{definition}
While various ways of detecting a \nearmeet exist (assuming a connection is possible~\citep{lavalle2006planning, ortiz2024idbrrt}, or trajectory optimization~\citep{choudhury2016regionally}),
we treat it as a black box.
\textit{After} this connection, we can't assume that the neighbours from either direction for the connected vertices will be identical, as the vertices aren't \textit{exactly} the same.
Thus, the techniques mentioned earlier for reducing the search effort \textit{after} the connection can't be used within our setting.
As the dynamically generated (sub)-trees in \ighastar are always subject to activation and approximate dominance, the $g$ value of the vertex is always an \textit{upper bound}, not a \textit{lower bound}; thus, it would make for an inadmissible heuristic and cannot be used in our setting.
\begin{algorithm}[t]
\small
\caption{\small Bidirectional \ighastar~(\biighastar)
}
\label{alg:bi-igha}
\textbf{Input:} Start \textcolor[HTML]{0AA844}{set $V_s$}, Goal set $V_g$, Resolution sequence $\mathcal{R}$\\
\textbf{Output:} Least-cost path or $\emptyset$ if \textsc{Failure}
\begin{algorithmic}[1]
\State $\hat{\pi} \gets \emptyset$\Comment{Shared amongst the two searches}
\Statex \textit{\textcolor[HTML]{0AA844}{Procedure for $D$ search; $\bar{D}$ search runs in the background}}
\State $v^D_r\text{.Active} \gets \textbf{true}$
    \State $Q^D_v \gets \{v^D_r\}$\ \Comment{Priority: $v.g + h_D(v)$}
    \State $l_D \gets 0$, $l'_D \gets 0$
    \While{$Q^D_v\neq\emptyset$}\Comment{\textbf{\textit{Iterate}} over $D$ search}
        \State \activate()\label{ll:igha_activate}
        \While{$Q^D_v.\text{Active} \neq \emptyset$}\label{ll:igha_search_loop_start}
            \If{$Q^D_v.\text{Active.}\textsc{peek}().f \geq w(\hat{\pi})$}
                \textbf{break}\label{ll:igha_peek_active}
            \EndIf
            \If{$Q^D_v.\text{Active.}\textsc{peek}() \in V^D_T$}\label{ll:igha_goal_check}
                \State $\hat{\pi} \gets \textsc{EmitPath}()$\label{ll:igha_emit_path}
                \State \textbf{break}
            \EndIf
            \If{\textsc{Shift}()}\label{ll:shift_call}
                \textbf{break}\Comment{Internally update $l'$}
            \EndIf
            \State $u \gets Q^D_v.\text{Active.}\textsc{pop}()$\label{ll:igha_popactive}
            \ForAll{$v \in \Succ(u)$}
                \If{$v.g < \hat{v}(v, R_l).g$}
                    \State $\hat{v}(v, R_l)\text{.Active} \gets \textbf{false}$\label{ll:igha-dom-deactivate}
                    \State $\hat{v}(v, R_l) \gets v$
                    \State $v\text{.Active} \gets \textbf{true}$
                \Else
                    \State $v\text{.Active} \gets \textbf{false}$\label{ll:igha-deactivate}
                \EndIf
                \State $Q^D_v.\textsc{Insert}(v)$\label{ll:igha_Q_v_save}
                \textcolor[HTML]{0AA844}
                {   
                    \If{$\nearmeet(v, R_\text{LCR}, T_{\bar{D}})$}\label{ll:biigha_near_meet_check}
                        \State $\bar{v} \gets \textsc{GetNearMeetVertex}()$\label{ll:biigha_get_near_meet_vertex}
                        \State $w(\pi_\text{LCR}) \gets v.g + \bar{v}.g + w(v, \bar{v})$\label{ll:get_near_meet_cost}
                        \If{$w(\pi_\text{LCR}) < w(\hat{\pi})$}\label{ll:biigha_near_meet_monotonic_improvement}
                            \State $\hat{\pi} \gets \textsc{EmitPath}()$\label{ll:biigha_emit_near_meet_path}
                        \EndIf
                    \EndIf
                }
            \EndFor
        \EndWhile
        \State $\textsc{Bound}(w(\hat{\pi}))$\Comment{$\forall v\in Q^D_v$ Remove $v:v.f>w(\hat{\pi})$}\label{ll:igha_bound}
        \State $\textsc{Project}(R_l)$\Comment{$l_D \gets l'_D$, update $R_l$, and recompute $\hat{v}$}\label{ll:igha_project}
    \EndWhile
\State \Return $\hat{\pi}$
\end{algorithmic}
\end{algorithm}
\begin{figure}
    \centering
    \includegraphics[width=0.9\linewidth]{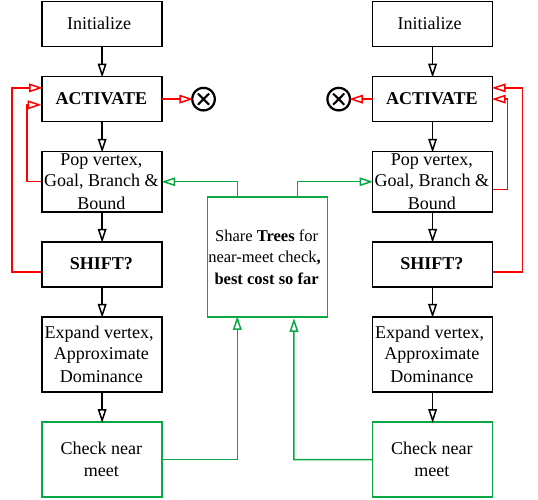}
    \caption{\biighastar runs two anti-parallel \ighastar searches with information sharing while checking for near-meets between the trees of them}
    \label{fig:flowchart}
\end{figure}
\subsection{The \biighastar Framework}\label{subsec:biigha_framework}
The Bidirectional \ighastar~(\biighastar)~(Alg.~\ref{alg:bi-igha}, Fig.~\ref{fig:flowchart}) runs two (\forward,~\backward) \ighastar searches that 
share tree information for detecting near meets and the current best cost $w(\hat{\pi})$.
This information sharing is an ever-present background process.
For simplicity, we pop vertices from the \forward and \backward $Q^D_v$s symmetrically.
\begin{definition}[\forward-\backward-Paths]\label{fb_paths}
When \nearmeet$(v, R_\text{LCR}, T_{\bar{D}})$ returns true for a vertex $v\in V_D$ with $\bar{v}\in V_{\bar{D}}$, (Lines~\ref{ll:biigha_near_meet_check}), we call the resulting path a \forward-\backward path denoted as $\pi_\text{FB}$, and denote the set of such paths as $\Pi_\text{FB}$.
\end{definition}
When there exist multiple $\pi_\text{FB}$ for a vertex $v \in V_D$ and multiple opposing vertices 
$V_\nearmeet \subseteq V_{\bar{D}}$,
\gnmv returns the vertex
$
\bar{v}^* = \argmin_{\bar{v} \in V_\nearmeet}
\big( v.g + \bar{v}.g + w(v,\bar{v}) \big),
$
i.e., the one minimizing the \forward-\backward path cost for
$\pi_\text{LCR}\in \Pi_\text{FB}$~(Line~\ref{ll:get_near_meet_cost}).
Note that the \forward and \backward \ighastar are not level-locked; they can change levels independently.
If one terminates before the other, the remainder continues running to termination, as a path may exist in one tree that does not exist in the other direction's tree.
Note that both searches share $\hat{\pi}$.
One \ighastar does not directly affect the other {\ighastar}'s \shift or \activate.

\textbf{Our key insight} is that bidirectionality alters how goal-reachability information propagates through the incremental dominance structure, mitigating resolution-induced freezing effects; we term this phenomenon Bidirectional Mitigation, and illustrate it through Fig.~\ref{fig:main_explanation}.
When it comes to \biighastar, even when the resolution is too coarse for \ighastar to find any path from $v_s$ to $V_g$ in the forward tree, requiring resolution refinement,
~\biighastar may find a path from $v_s$ to $v_g$ through $\Pi_\text{FB}$~(Fig.~\ref{fig:main_explanation}).
\textit{Bidirectional Mitigation} allows at least one of the searches in \biighastar to \textit{encounter} an \textit{equivalent} cost solution at a lower resolution.
At lower resolutions, the effective branching factor of \biighastar's \forward and \backward searches can be much lower via approximate dominance, which also contributes to the speed-up.
We analyse this formally in Sec.~\ref{subsec: freezing_theory}, show how this results in better-than-expected reduction in the effective branching factor in Sec.~\ref{sec: surprise}, and
show empirically in Sec.~\ref{sec: experiments} that this~(mitigation) happens very often when there is a speed-up.
\textbf{Note}: we use \textbf{\forward} and \textbf{\backward} to refer to \biighastar's internal searches, and \textbf{forward/reverse} to refer to instances of~(unidirectional) \ighastar.

\section{Formal Guarantees and Structural Analysis}
\label{sec:bi-igha-theory}
We first show that Bi-\ighastar preserves the core guarantees of \ighastar on monotonic cost improvement and termination with finite expansions.
We then formalize the structural interaction between discretization and freezing that motivates bidirectionality.
Throughout this section we assume:

\begin{enumerate}[label=\textbf{A\arabic*}]
    \item 
    The heuristic $h_D$ never overestimates the true cost-to-go to $V^D_T$\label{heuristic_assumption}.
    \item 
    Each vertex has at most $b<\infty$ successors under $\Succ$\label{tree_assumption}.
    \item 
    Every edge cost satisfies $w(u,v)\ge \epsilon > 0$\label{positive_edge_assumption}.
    \item Every \backward edge corresponds to a \forward realizable edge of equal cost
    \label{reversibility_assumption}.
\end{enumerate}

\ref{heuristic_assumption}-\ref{positive_edge_assumption} are identical to the assumptions in~\citep{talia2025incremental}.
From~\citep{talia2025incremental}, the (relevant) invariants of \ighastar are:
\begin{enumerate}[label=\textbf{I\arabic*}]
    \item Only \text{Active} vertices are expanded during search.\label{inv:active_only}
    \item At least one vertex is {\activate}d until $Q^D_v=\emptyset$.\label{inv:exp_one_vert}
\end{enumerate}

\subsection{Preserving \ighastar's Guarantees}
\begin{theorem}[\ighastar Preservation]\label{thm:igha_equal}
    Given \ref{heuristic_assumption}-\ref{reversibility_assumption}, \ref{inv:active_only}-\ref{inv:exp_one_vert}, \biighastar retains \ighastar's guarantees of monotonic cost improvement and termination with finite expansions.
\end{theorem}
\begin{proof}(Sketch)
    If a path is emitted from the \forward or \backward search, 
    either by reaching their respective $V^D_T$~(Line~\ref{ll:igha_emit_path}) or a \nearmeet~(Line~\ref{ll:biigha_emit_near_meet_path}),
    then this is because the search found a valid path that is guaranteed to be better than the existing best estimate~$w(\hat{\pi})$~(Lines~\ref{ll:igha_peek_active},~\ref{ll:biigha_near_meet_monotonic_improvement}).
    Given this, it follows from \ref{heuristic_assumption}-\ref{reversibility_assumption} and \ref{inv:active_only},~\ref{inv:exp_one_vert} that \biighastar retains \ighastar's guarantees of monotonic cost improvement and termination with finite expansions.
\end{proof}

\subsection{Freezing as a Structural Barrier}\label{subsec: freezing_theory}
\begin{observation}[Frozen-Vertex Barrier]\label{obs:frozen_barrier}
~\ref{inv:active_only} implies that 
if a vertex is frozen ($v.\text{Active}=\text{false}$)
at a given iteration
then no descendant of $v$ can be generated at that iteration;
consequently, no solution path passing through $v$ can be discovered at that iteration.
\end{observation}
We \textbf{construct} a family of planning instances for the forward search that expose the frozen vertex barrier ~(Obs~\ref{obs:frozen_barrier}) when there exists a \textit{bottleneck vertex} and the search does not have access to information about whether the vertex is on the solution path or not (the search is naive).
\begin{definition}[Bottleneck Vertex]
    A vertex $v_b$ is considered a bottleneck if every path from the root vertex $v^D_r$ to the target set $V^D_T$ must traverse through it.
\end{definition}
\begin{definition}(Naive Rule)
    A Naive Rule~\citep{talia2025incremental} refers to a combination of a naive \shift and a naive \activate method, where either method is considered naive if it only uses information in $Q^D_v$.
\end{definition}\label{def:naive_rule}

\begin{definition}[Bottleneck Family $\{\mathcal{P}_k\}$]\label{def:bottleneck_family}
For a unidirectional \ighastar with a Naive Rule, for each $k \ge 1$, let $\mathcal{P}_k$ denote a planning instance with $k$ sequential bottleneck vertices $v^k_b$, such that another vertex $v^i_\text{dead}=\hat{v}(v^i_b,R_i)$, at level $i$~(See~\ref{sec:igha_overview}),  dominates it.
Here, $v^i_\text{dead}$ does not lead to $V^D_T$, but rather to a dead-end subtree containing $M_i>M\forall i$ vertices, none of which lead to $V^D_T$, and $M$ is an arbitrarily large constant.
Finally, for a naive \biighastar, assume that at any iteration $|\Pi_\text{FB}|\geq 1$, where the 
\forward-\backward 
paths that do not require forward expansion of $v^0_b$ and require O(1) expansions beyond those required to reach $v^0_b$.
\end{definition}

\begin{observation}[Unidirectional Freezing Blowup]
\label{thm:forward_blowup}
On $\mathcal{P}_k$, it follows from Obs.~\ref{obs:frozen_barrier} that
for $k$ bottlenecks, \ighastar expands at least $\Omega(M_0)$ expansions if \shift refines to the highest possible resolution after the first iteration and at least $\Omega(\sum^k_{i=0}M_i)$ if not~(from Inv.~\ref{inv:exp_one_vert}).
\end{observation}

The sequence of plots in Fig.~\ref{fig:igha_anecdote_3DoF_0}-\ref{fig:igha_anecdote_3DoF_7} illustrates \textbf{this blowup for} $\mathbf{k=1}$, where
\ighastar's search tree passes near the goal region, but freezes a crucial vertex leading to the goal, explores all the dead-end nodes before refining resolution and finally expanding the crucial vertex.
\textbf{The blowup arises from freezing, not from search depth alone.}

\begin{observation}
On $\mathcal{P}_k$, through $\Pi_\text{FB}$'s found via \nearmeet, \biighastar finds a path from $v_s$ to $v_g$ in $O(1)$ additional expansions beyond reaching $v^0_b$.
\end{observation}

\begin{corollary}[Structural Separation]
\label{cor:separation}
There exists a family of 
planning 
instances for which unidirectional \ighastar may require at least $\Omega\!\left(\sum_i M_i\right)$ expansions, while \biighastar requires $O(1)$ additional expansions beyond reaching $v^0_b$.
\end{corollary}
This separation shows that the improvement of \biighastar is not solely due to reduced effective depth,
and we will empirically show in Sec.~\ref{sec: surprise} that this separation exists in practice.

\begin{figure}[t]
  \centering
  \includegraphics[width=\linewidth]{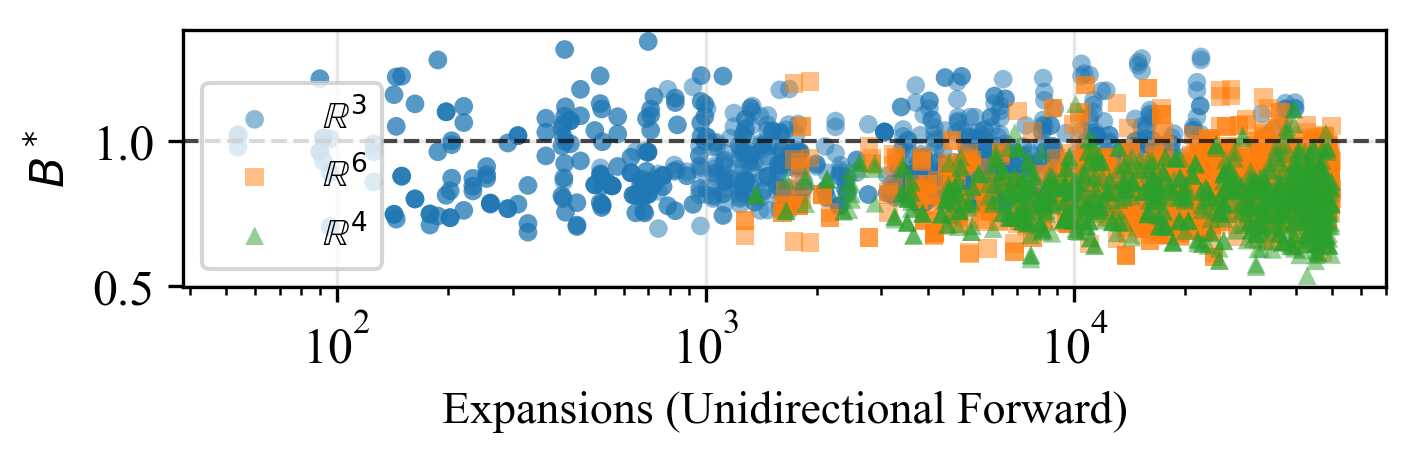}
\caption{
$B^*$ (see Sec.~\ref{sec: surprise}) scatter plot for \biighastar vs. \ighastar w.r.t. forward expansions, for best cost path.
We usually expect $B^*\leq1$; however,
through bidirectional mitigation, \biighastar \textit{appears} to cross this limit.
}
\label{fig:b*_ratio}
\end{figure}
\subsection{Surprising Branching factors}\label{sec: surprise}
We defer experimental details to Sec.~\ref{sec: experiments}, and \textbf{\textit{empirically}} illustrate 
the structural separation~(Corr.~\ref{cor:separation}) of Bidirectional mitigation.
Across $\mathbb{R}^3$, $\mathbb{R}^4$, and $\mathbb{R}^6$ systems, Fig.~\ref{fig:R3R4R6_speedup} shows that, as intuition may suggest, increasing $R_\text{LCR}$ raises the likelihood of near-meets and thus the speed-up, until we examine effective branching factors.

We define the \emph{empirical effective unidirectional branching factor} $b_{\text{uni}}$ for a given run of a unidirectional search as follows:
Let 
$E_\text{uni}$ be the number of node expansions
and 
$d_\text{uni}$ be the number of nodes of the path returned.
Then, $b_{\text{uni}}$ is the solution to the following equation:
$E_\text{uni}=\sum^{n=d_\text{uni}-1}_{n=0} (b_\text{uni})^n$.
Given $b_{\text{uni}}$, we can compute $E^\text{ideal}_\text{bi}$, the number of expected node expansion under an optimal symmetric bidirectional search.
Specifically, 
$E^\text{ideal}_\text{bi}:=2\sum^{n=d_\text{uni}/2-1}_{n=0}(b_\text{uni})^n$.
Given $E^\text{ideal}_\text{bi}$, we define the \emph{estimated effective bidirectional branching factor}
as the solution to the following equation:
$E^\text{ideal}_\text{bi}=\sum^{n=d_\text{uni}-1}_{n=0} (b^\text{ideal}_\text{bi})^n$.
In a completely analogous way to the unidirectional setting,  we  define the \emph{empirical effective bidirectional branching factor}~$b_{\text{bi}}$ for a given run of a bidirectional search as follows:
Let~$E_\text{bi}$ be the number of node expansions
and 
$d_\text{bi}$ be the number of nodes of the path returned.
Then, $b_{\text{bi}}$ is the solution to the following equation:
$E_\text{bi}=\sum^{n=d_\text{bi}-1}_{n=0} (b_\text{bi})^n$.

Now, consider the ratio $B^*=\frac{b^\text{ideal}_\text{bi}}{b_\text{bi}}$.
Since $b^\text{ideal}_{\text{bi}}$ assumes an optimal symmetric bidirectional search, we expect that~$B^* \leq 1$.
Surprisingly, empirical values recorded and reported in Fig.~\ref{fig:b*_ratio} demonstrate instances for which 
$B^* > 1$.
The reason for this phenomenon is that in the computation of~$B^*$, we implicitly assumed that both planners operate on the same search structure; this is typically true for search on a static graph, but is not the case here due to mitigation of the frozen vertex barrier.
Nonetheless, as \biighastar does not explicitly enforce MMP, at a first glance of Alg.~\ref{alg:bi-igha} and flow chart Fig.~\ref{fig:flowchart}, one would \textit{never} expect this result from it.

\begin{figure*}[!htb]
\centering
\begin{subfigure}{0.16\linewidth}
  \includegraphics[width=\linewidth]{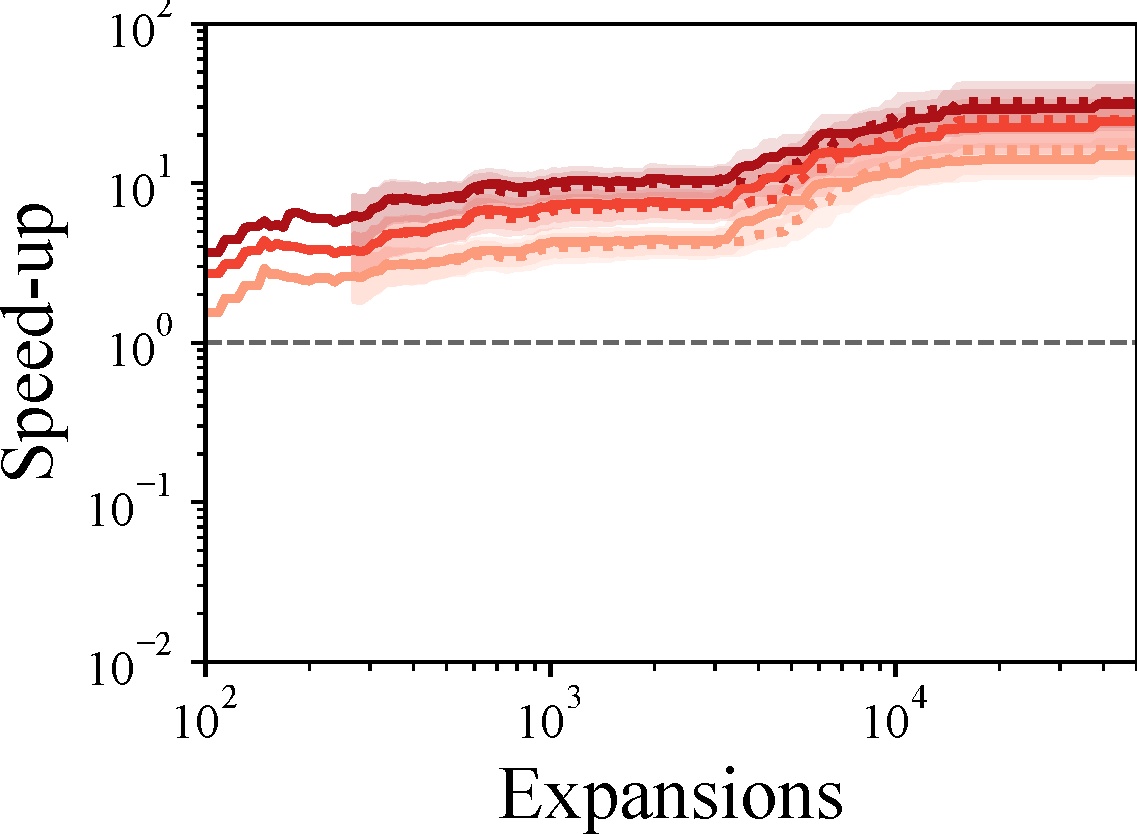}
  \caption{$\mathbb{R}^3$ For., First}
  \label{fig:kinematic_forward_only_first_path}
\end{subfigure}
\begin{subfigure}{0.16\linewidth}
  \includegraphics[width=\linewidth]{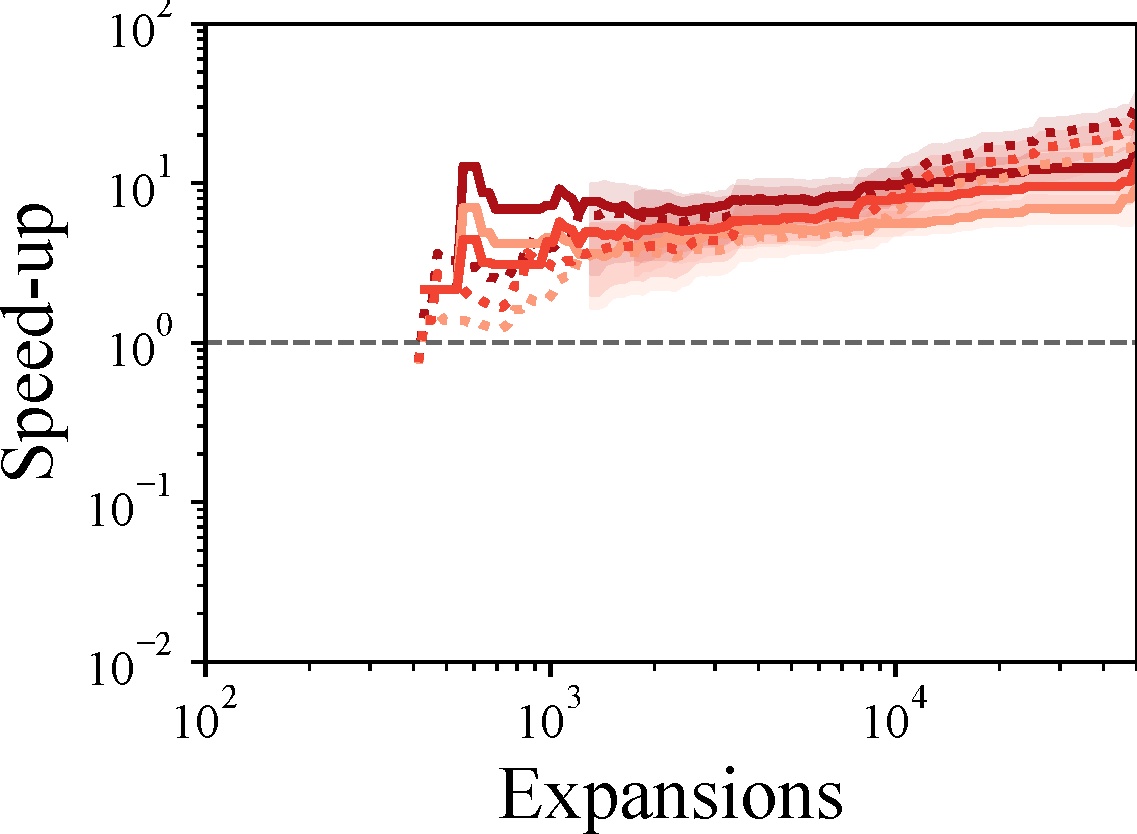}
  \caption{$\mathbb{R}^4$ For., First}
  \label{fig:kinodynamic_forward_only_first_path}
\end{subfigure}
\begin{subfigure}{0.16\linewidth}
  \includegraphics[width=\linewidth]{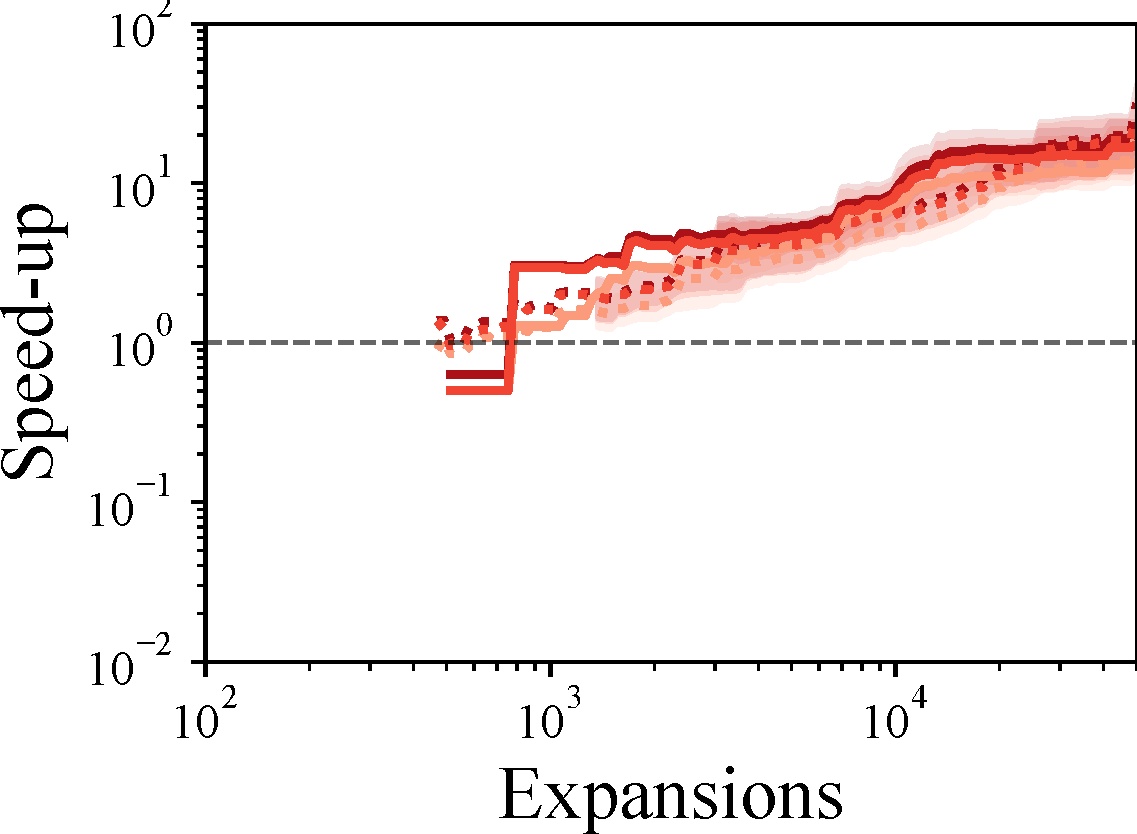}
  \caption{$\mathbb{R}^6$ For., First}
  \label{fig:hovercraft_forward_only_first_path}
\end{subfigure}
\begin{subfigure}{0.16\linewidth}
  \centering
  \includegraphics[width=\linewidth]{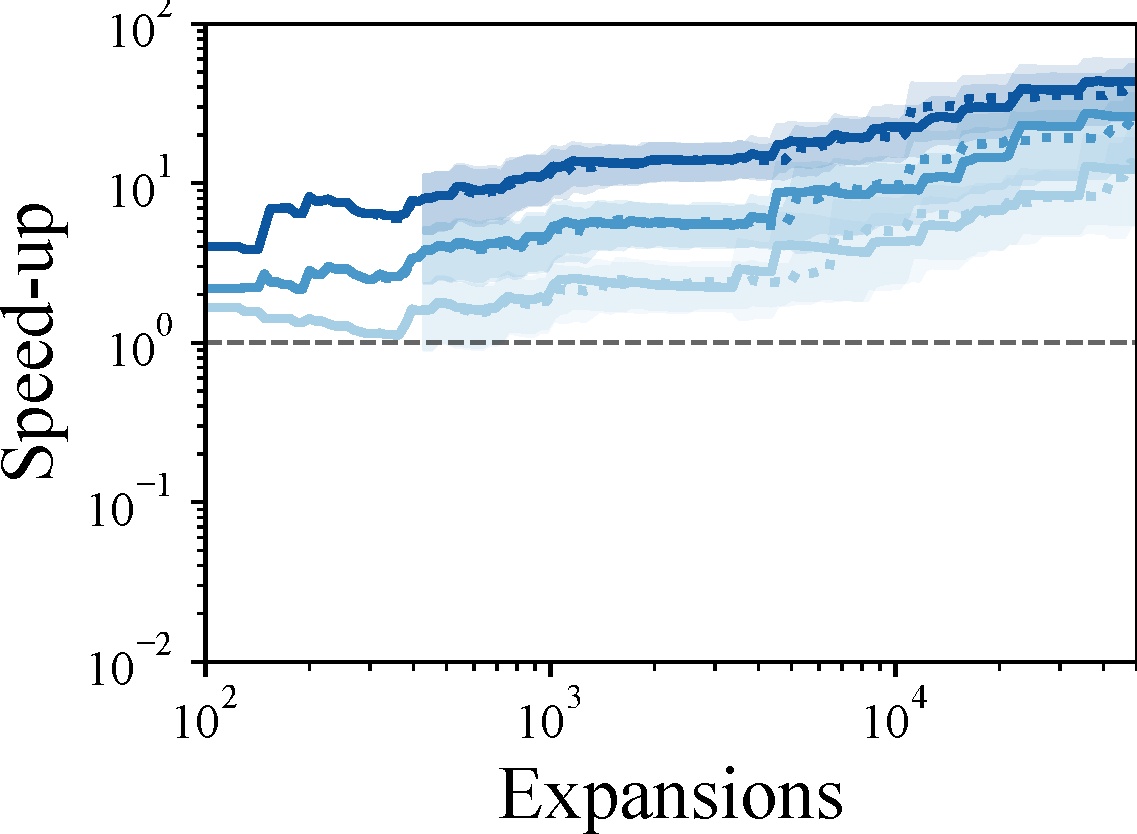}
  \caption{$\mathbb{R}^3$ For., Best}
  \label{fig:kinematic_forward_only_best_path}
\end{subfigure}
\begin{subfigure}{0.16\linewidth}
  \centering
  \includegraphics[width=\linewidth]{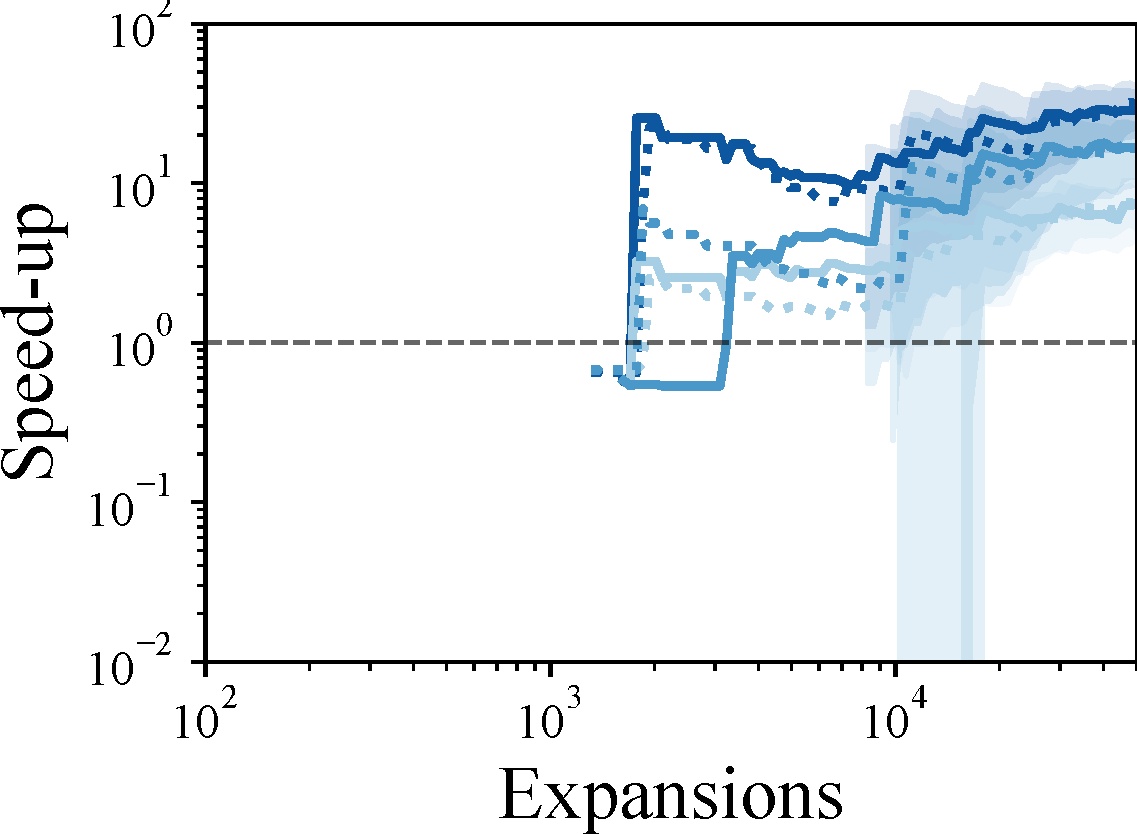}
  \caption{$\mathbb{R}^4$ For., Best}
  \label{fig:kinodynamic_forward_only_best_path}
\end{subfigure}
\begin{subfigure}{0.16\linewidth}
  \centering
  \includegraphics[width=\linewidth]{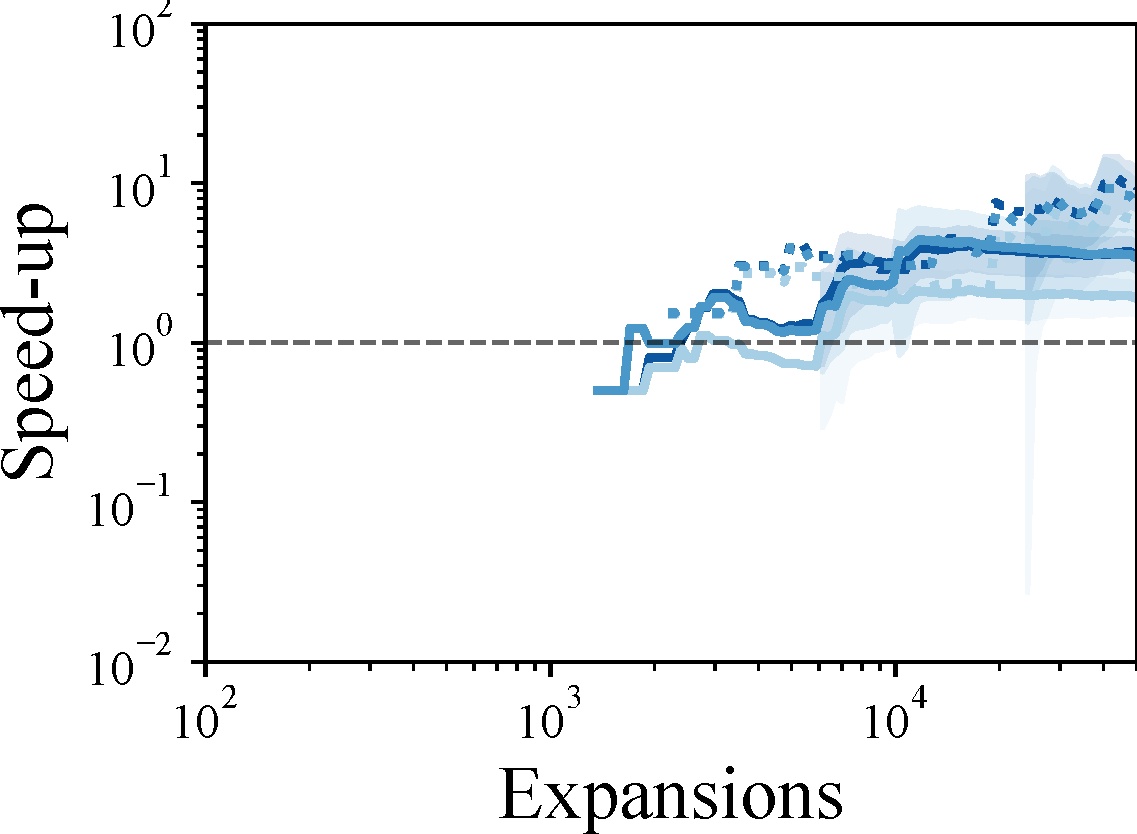}
  \caption{$\mathbb{R}^6$ For., Best}
  \label{fig:hovercraft_forward_only_best_path}
\end{subfigure}%
\\
\begin{subfigure}{0.16\linewidth}
  \includegraphics[width=\linewidth]{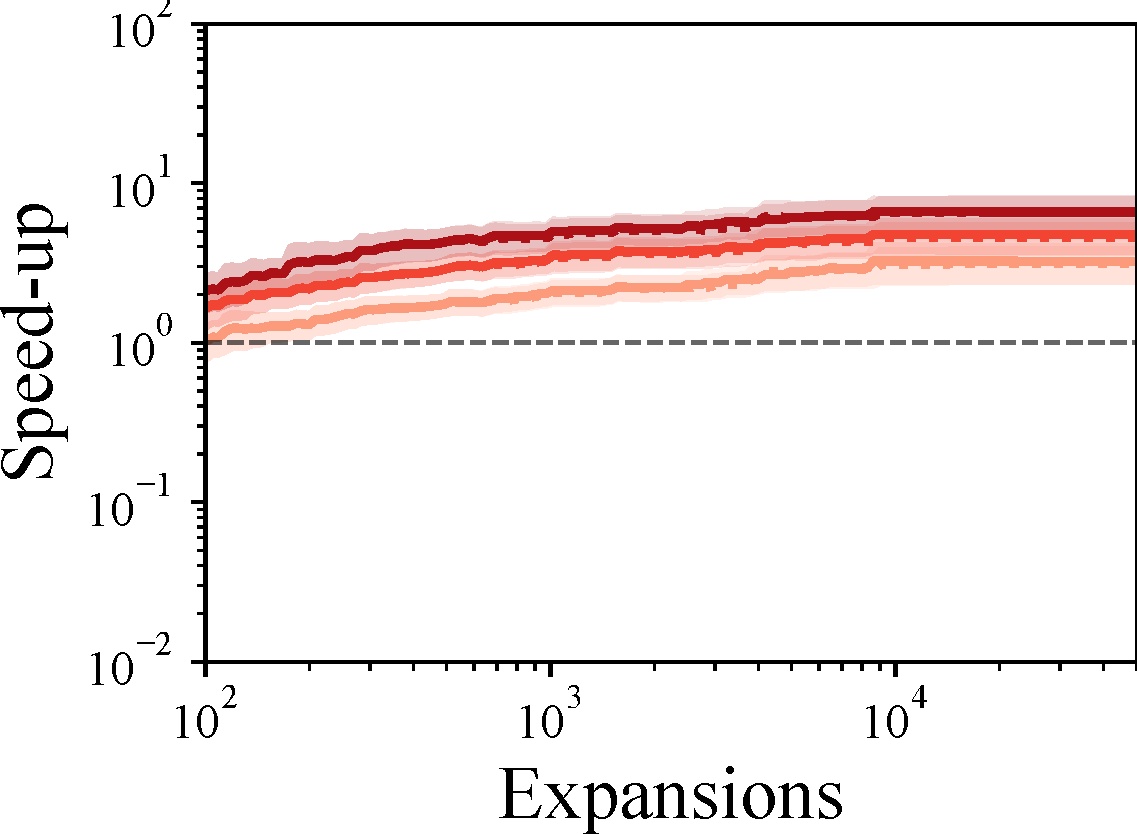}
  \caption{$\mathbb{R}^3$ Min, First}
  \label{fig:kinematic_min_first_path}
\end{subfigure}
\begin{subfigure}{0.16\linewidth}
  \includegraphics[width=\linewidth]{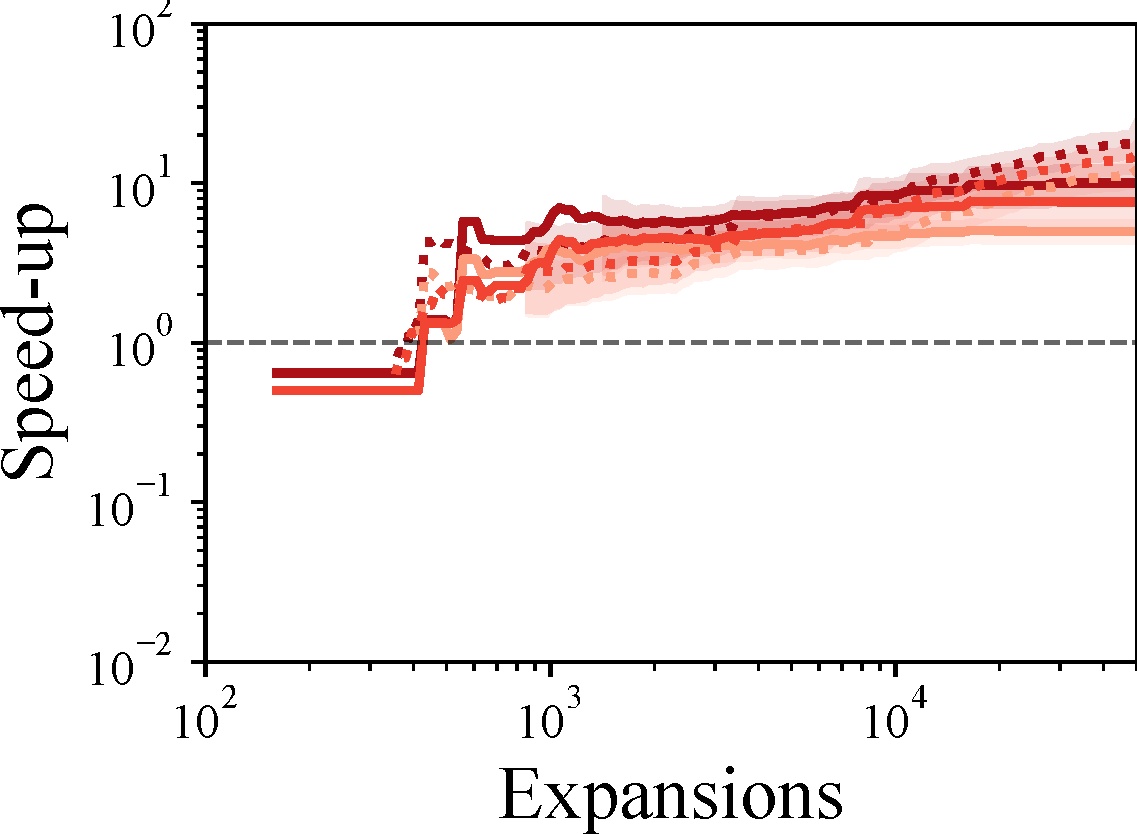}
  \caption{$\mathbb{R}^4$ Min, First}
  \label{fig:kinodynamic_min_first_path}
\end{subfigure}
\begin{subfigure}{0.16\linewidth}
  \includegraphics[width=\linewidth]{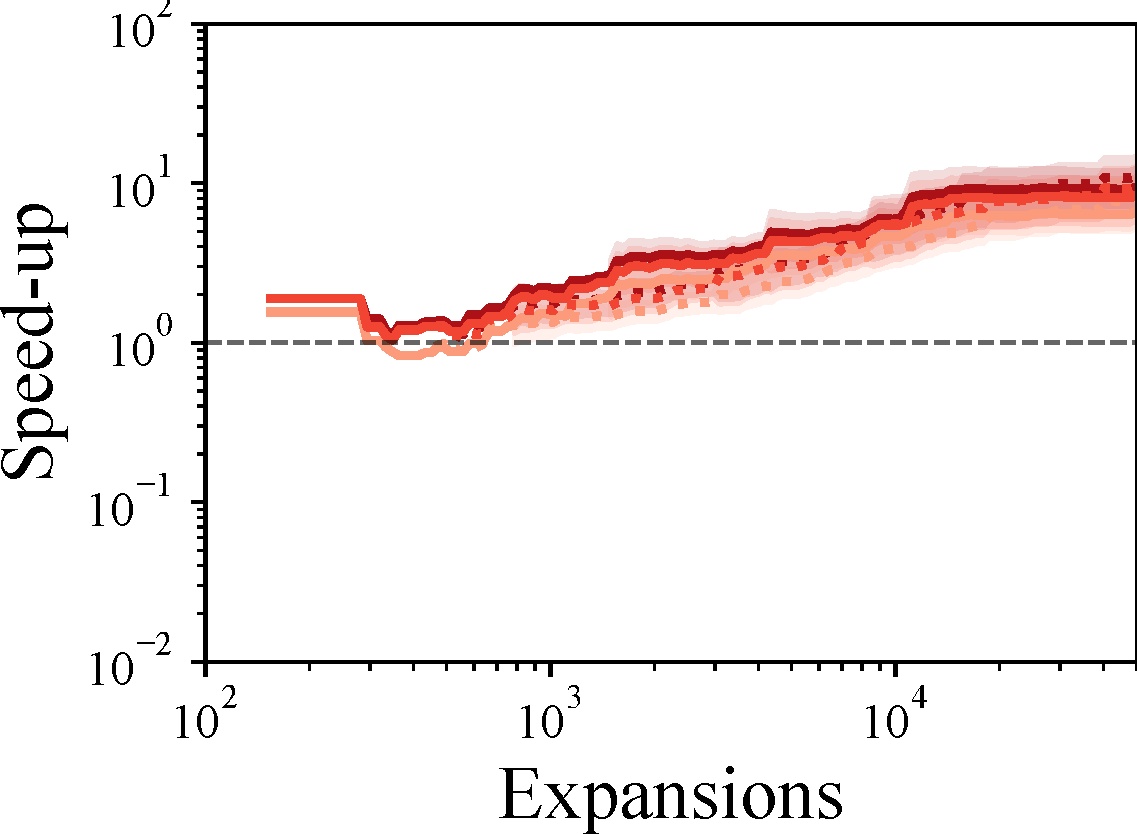}
  \caption{$\mathbb{R}^6$ Min, First}
  \label{fig:hovercraft_min_first_path}
\end{subfigure}
\begin{subfigure}{0.16\linewidth}
  \centering
  \includegraphics[width=\linewidth]{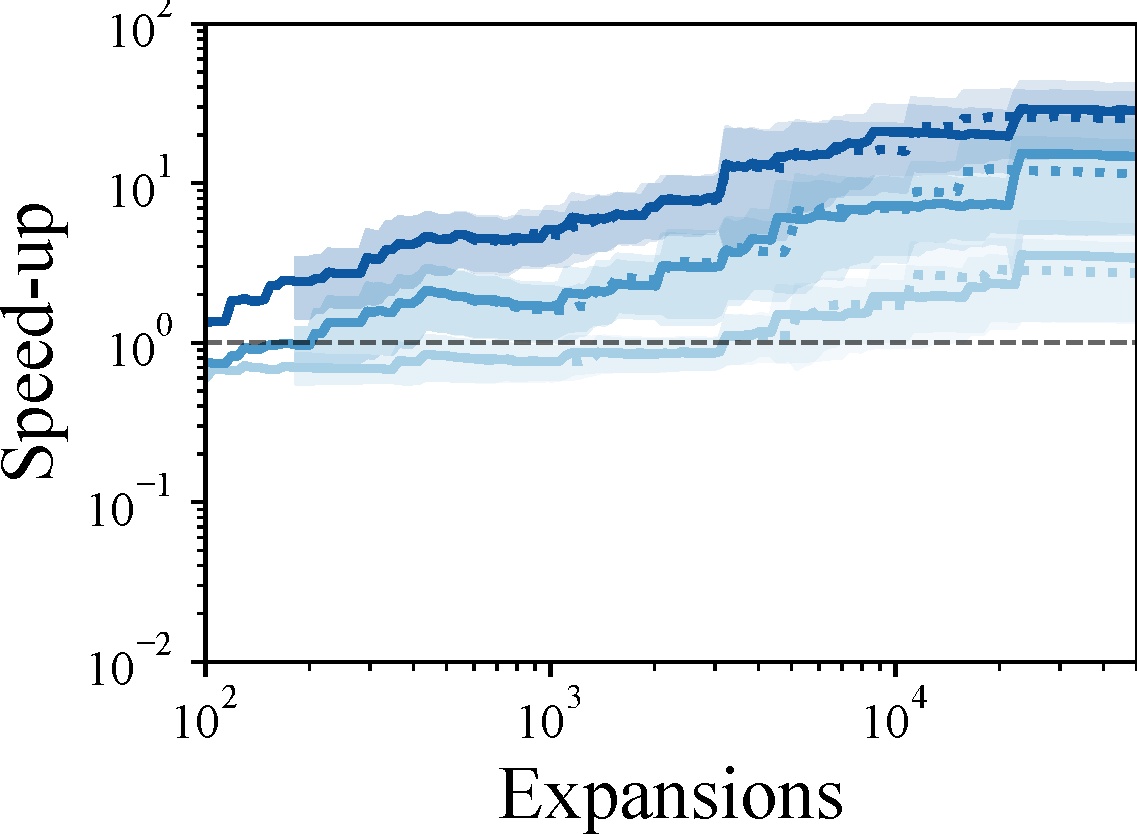}
  \caption{$\mathbb{R}^3$ Min, Best}
  \label{fig:kinematic_min_best_path}
\end{subfigure}
\begin{subfigure}{0.16\linewidth}
  \centering
  \includegraphics[width=\linewidth]{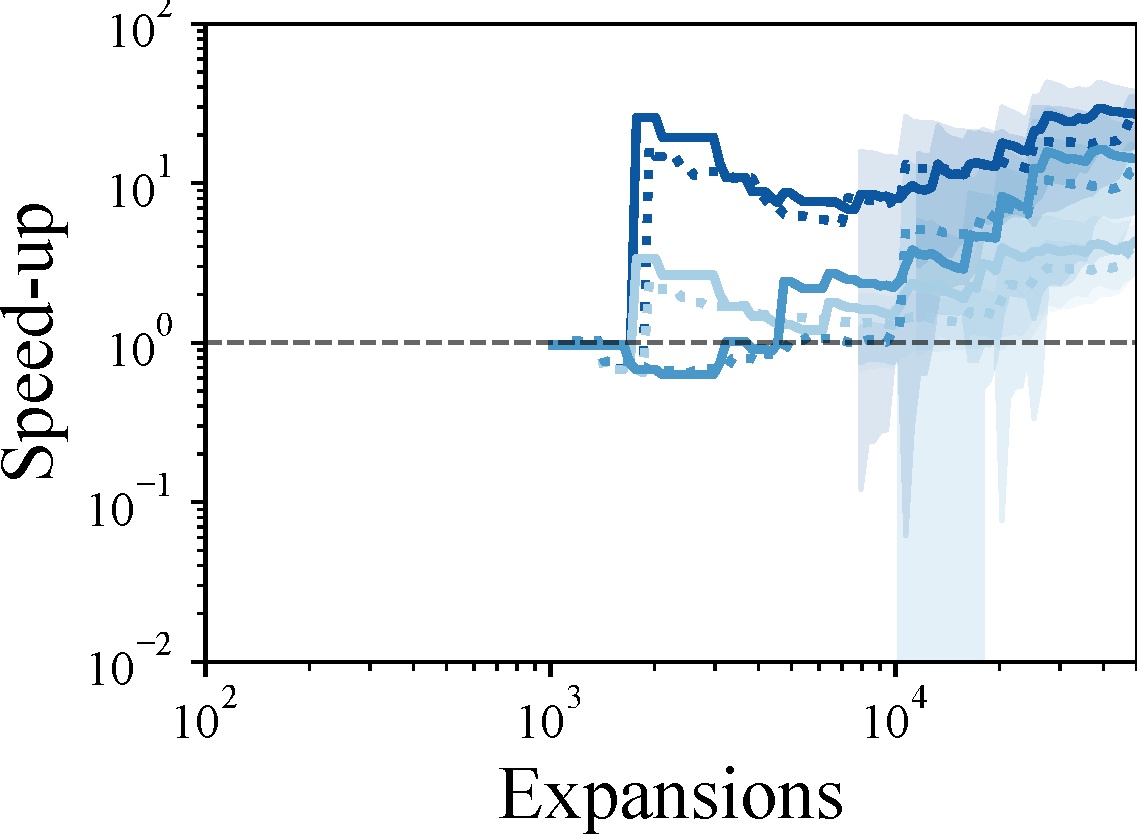}
  \caption{$\mathbb{R}^4$ Min, Best}
  \label{fig:kinodynamic_min_best_path}
\end{subfigure}
\begin{subfigure}{0.16\linewidth}
  \centering
  \includegraphics[width=\linewidth]{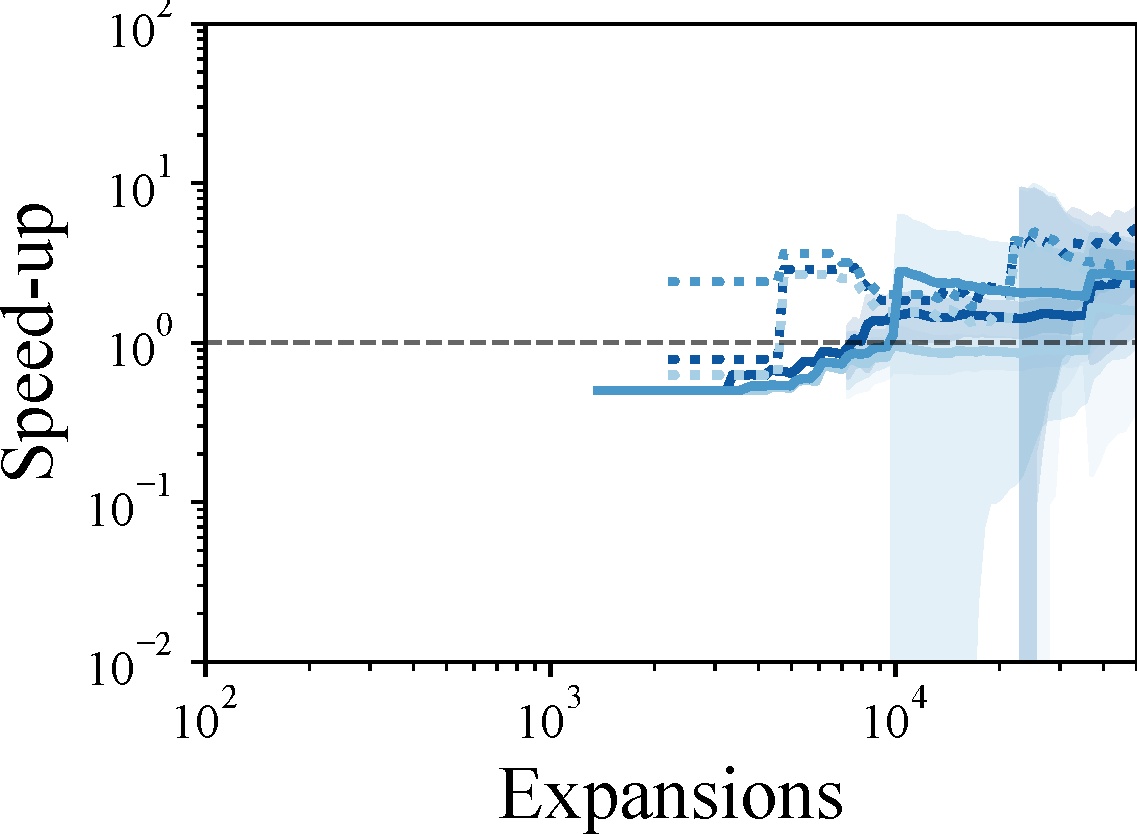}
  \caption{$\mathbb{R}^6$ Min, Best}
  \label{fig:hovercraft_min_best_path}
\end{subfigure}%
\caption{
Here, 
\legendsolid{0.675,0.063,0.086} denotes Speed-up to the first solution and 
\legendsolid{0.047,0.341,0.627} to best solution; lighter shades correspond to smaller LCR values (LCR, LCR/8, LCR/64). 
Solid lines compare \biighastar-$\infty$ vs.\ \ighastar-$\infty$, dotted/dashed lines compare \biighastar-$250$ vs.\ \ighastar-$250$. 
Confidence intervals are shown only when more than 30 measurements are available. 
Comparisons are against Forward-only (\textbf{For.})~(\subref{fig:kinematic_forward_only_first_path}-
\subref{fig:hovercraft_forward_only_best_path})
\ighastar\ and the faster (\textbf{Min}) of forward/reverse \ighastar~((\subref{fig:kinematic_min_first_path}-
\subref{fig:hovercraft_min_best_path}).
We observe roughly order-of-magnitude Speed-ups; gains reduce as LCR is reduced, when comparing to the Min., and increased complexity of higher dimensions shifts the curves right.
}
\label{fig:R3R4R6_speedup}
\vspace{-13pt}
\end{figure*}
\begin{figure*}
  \centering
  \includegraphics[width=\linewidth]{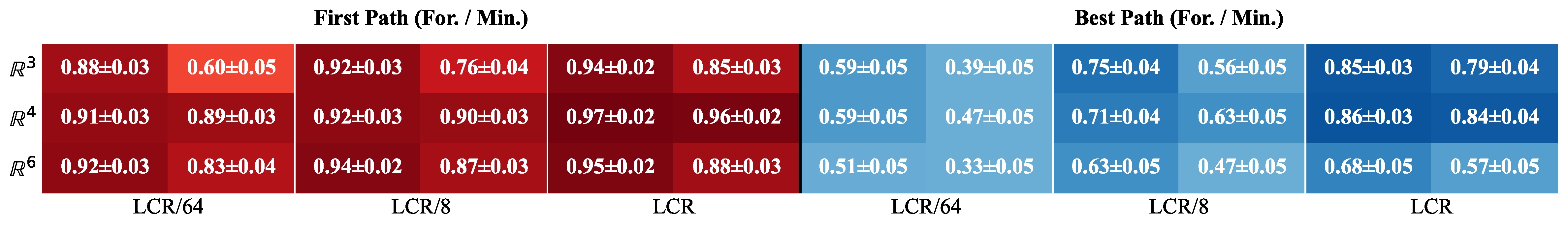}
\caption{$\hat{p}(\text{Speed-up} > 1)$ across LCR values are reported as mean $\pm$ 95\% CI, shades of red represent the values for the first path, and shades of blue represent the best path.
For each LCR, the left half represents the Speed-up probability compared to forward (\textbf{For.}) \ighastar, and the right half represents it against the faster of forward/reverse \ighastar (\textbf{Min}).
As LCR increases, $\hat{p}(\text{Speed-up} > 1)$ increases across the board.
}
\label{tab:speedup_prob_lcr}
\vspace{-20pt}
\end{figure*}
\begin{figure}
    \centering
      \includegraphics[width=\linewidth]{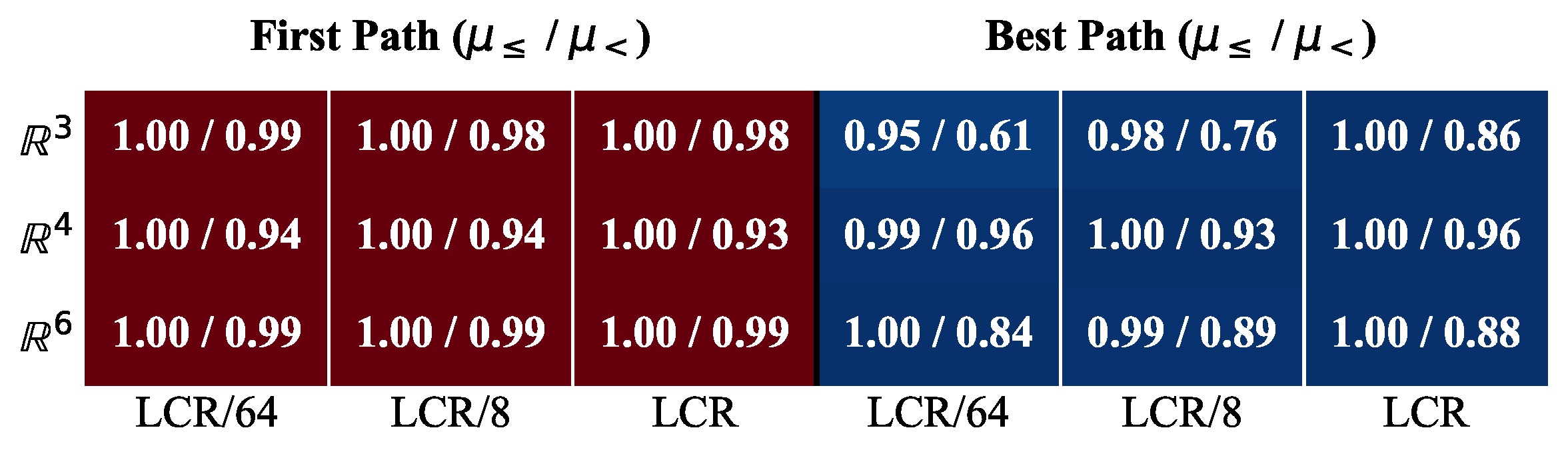}
    \caption{
    Showing $\mu_\leq,~\mu_<$ for $\mathbb{R}^N$ for the Min. of forward/reverse \ighastar; the likelihood that one of the \biighastar's resolution levels was lower than \ighastar's for the first or the best solution is almost always true.
    In the $\mathbb{R}^3$ setting, $\mu_<$ is lower because often the solution is available at a low enough resolution to both \biighastar and \ighastar.
    }
    \label{fig:level_probability_comparison}
\end{figure}

\begin{figure*}
    \centering
    \begin{subfigure}{0.16\linewidth}
      \centering
      \includegraphics[width=\linewidth]{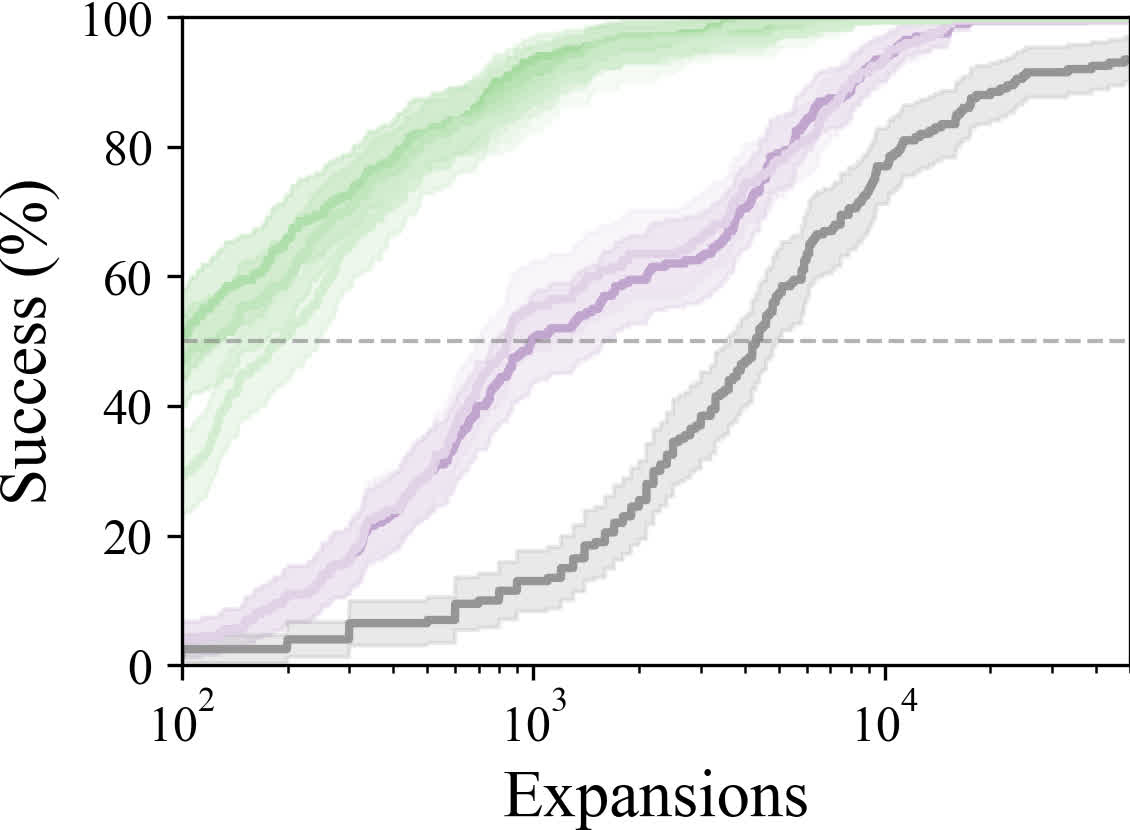}
      \caption{$\mathbb{R}^3$ Success}
      \label{fig:kinematic_success_plot}
    \end{subfigure}
    \begin{subfigure}{0.16\linewidth}
      \centering
      \includegraphics[width=\linewidth]{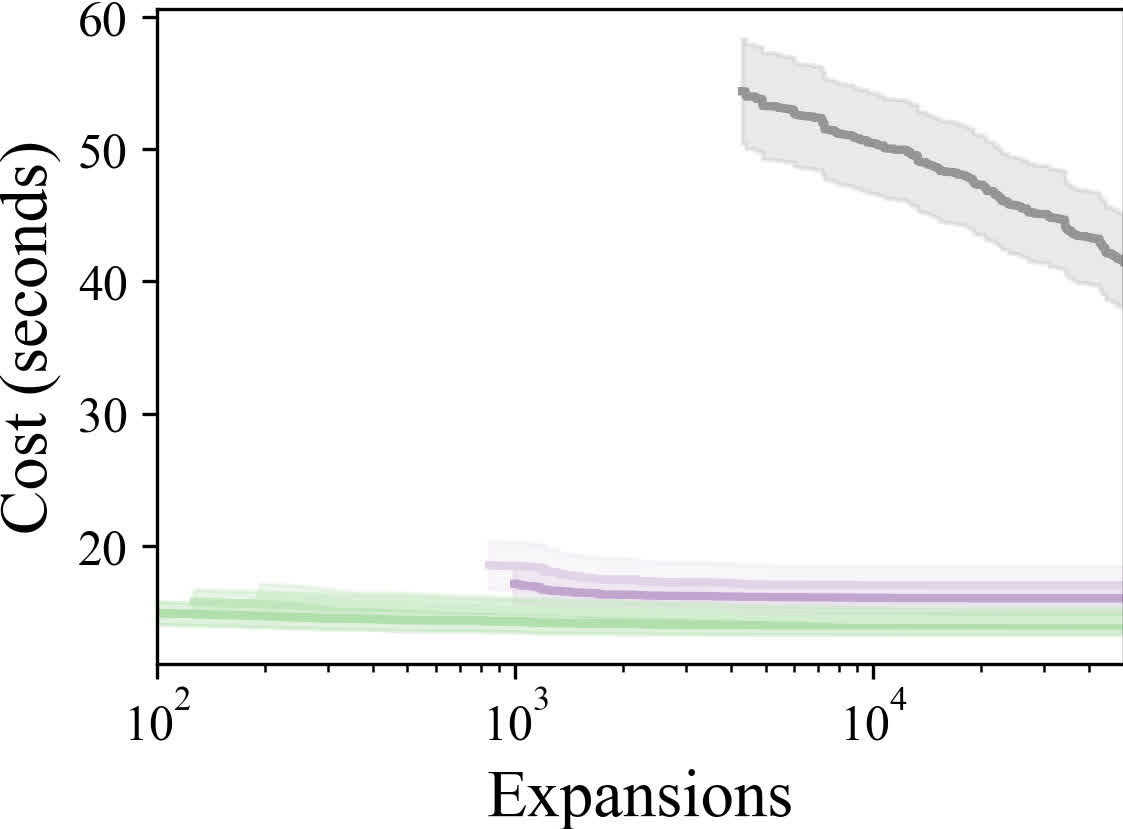}
      \caption{$\mathbb{R}^3$ Cost}
      \label{fig:kinematic_cost_plot}
    \end{subfigure}
    \begin{subfigure}{0.16\linewidth}
      \centering
      \includegraphics[width=\linewidth]{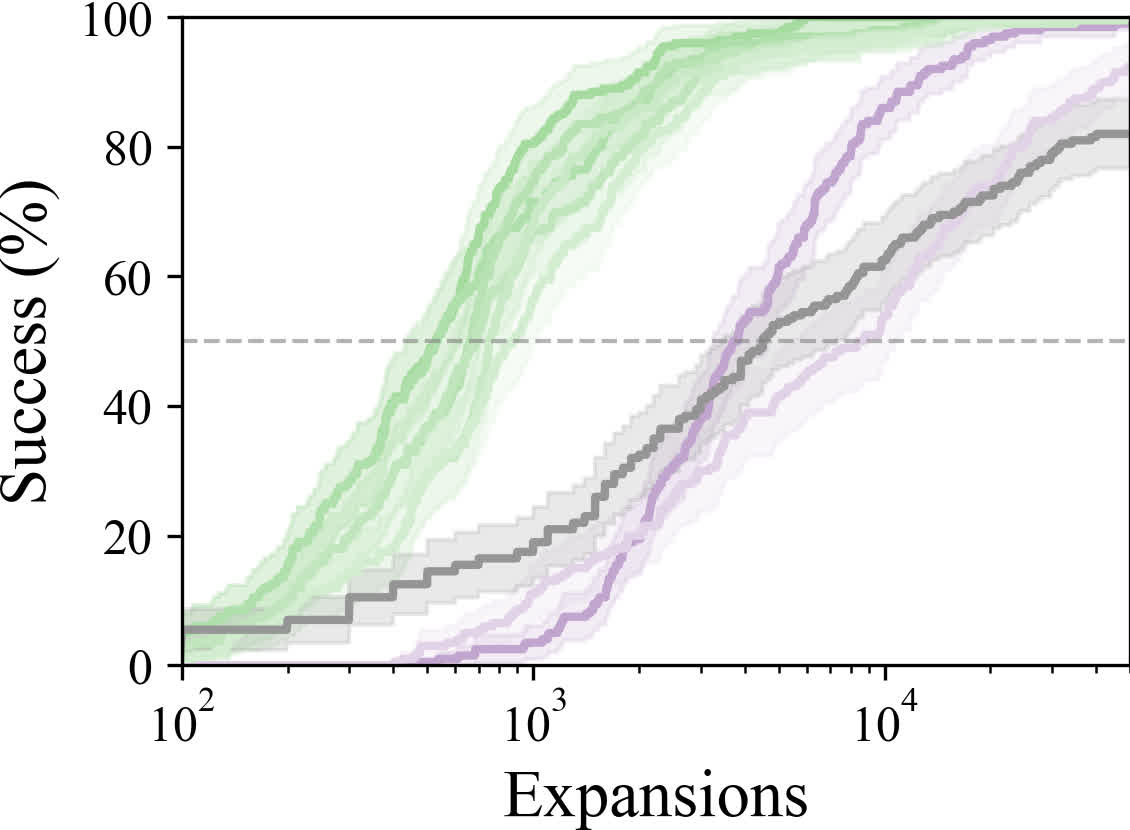}
      \caption{$\mathbb{R}^4$ Success}
      \label{fig:kinodynamic_success_plot}
    \end{subfigure}
    \begin{subfigure}{0.16\linewidth}
      \centering
      \includegraphics[width=\linewidth]{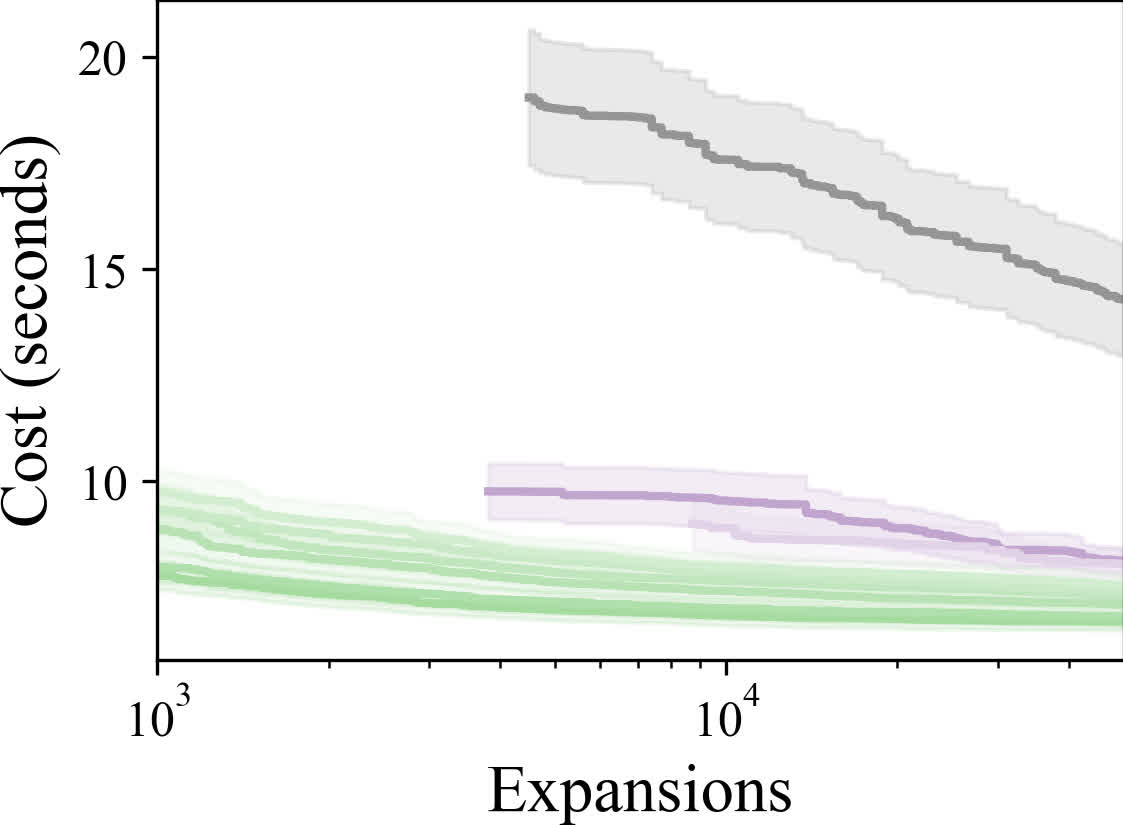}
      \caption{$\mathbb{R}^4$ Cost}
      \label{fig:kinodynamic_cost_plot}
    \end{subfigure}
    \begin{subfigure}{0.16\linewidth}
      \centering
      \includegraphics[width=\linewidth]{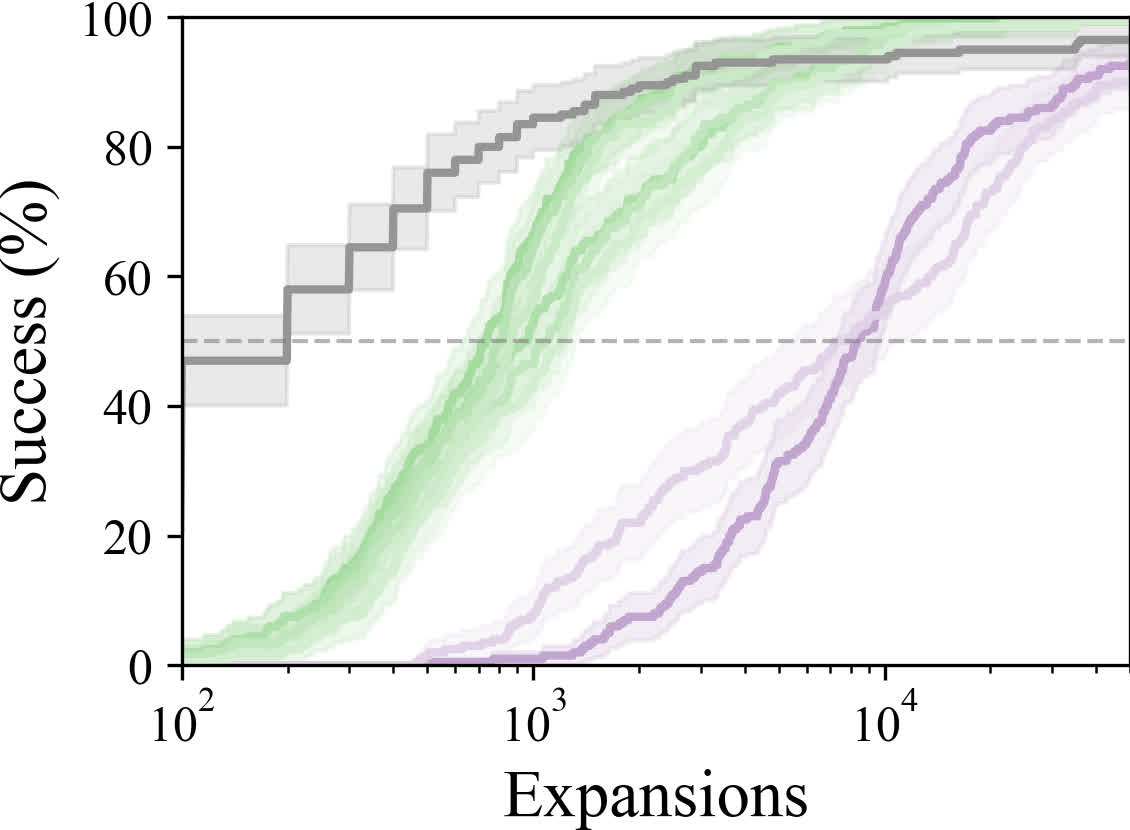}
      \caption{$\mathbb{R}^6$ Success}
      \label{fig:hovercraft_success_plot}
    \end{subfigure}
    \begin{subfigure}{0.16\linewidth}
      \centering
      \includegraphics[width=\linewidth]{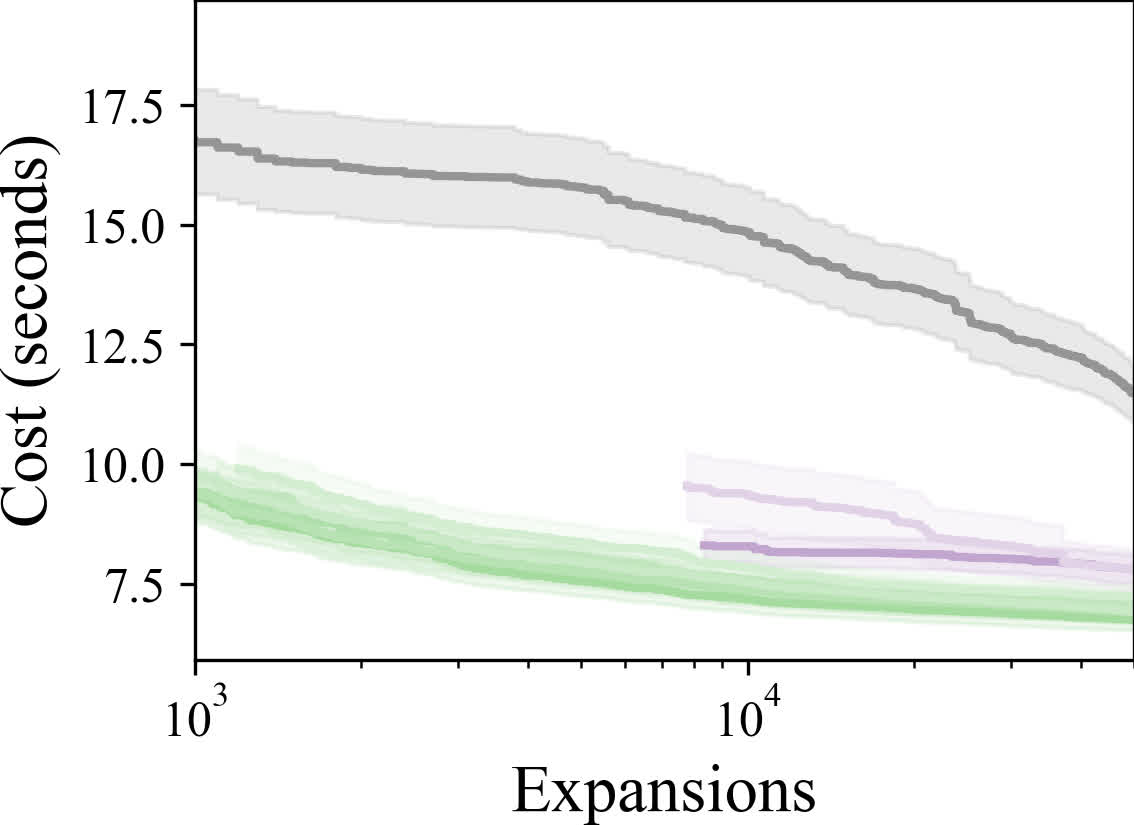}
      \caption{$\mathbb{R}^6$ Cost}
      \label{fig:hovercraft_cost_plot}
    \end{subfigure}%
    \caption{
    Here,
    \legendsolid{0.651,0.859,0.627}: Bi-\ighastar, 
    \legendsolid{0.761,0.647,0.812}: \ighastar, where darker shades imply bigger LCR,
    and additionally
    \legendsolid{0.5,0.5,0.5}: \sst (from OMPL~\citep{ompl}) as a reference point. 
    Mean cost is computed for problems succeeding in the first 50~\%.
    In Fig.~\ref{fig:R3R4R6_speedup},~\ref{tab:speedup_prob_lcr}, results are \textit{somewhat} pessimistic because we measure the expansions at which we get the same or better cost -- often the cost found by \biighastar is higher but very close ($<\epsilon$, recall~\ref{positive_edge_assumption}), to \ighastar's -- this plot shows that \biighastar's cost is consistently a left-shifted version of \ighastar's by $\approx$ an order of magnitude.
    We show \sst as an additional reference--\sst is good at rapidly exploring high-dimensional spaces; however, in mission-critical settings, solution quality dominates in-the-loop performance~(Fig.~\ref{fig:playback_success}-\ref{fig:in_the_loop_result} and~\citep{talia2025incremental}).
    }
    \label{fig:Success_rate_cost_comparison_open_loop}
    \vspace{-10pt}
\end{figure*}

\begin{figure*}[!htb]
\centering
\begin{subfigure}{0.16\linewidth}
  \includegraphics[width=\linewidth]{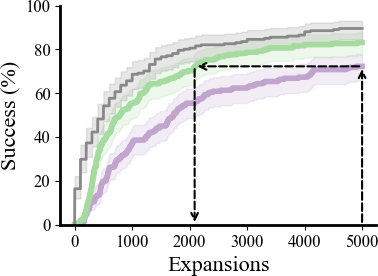}
  \caption{Playback success rate}
  \label{fig:playback_success}
\end{subfigure}
\begin{subfigure}{0.16\linewidth}
  \includegraphics[width=\linewidth]{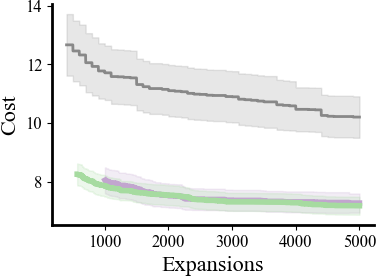}
  \caption{Playback cost}
  \label{fig:playback_cost}
\end{subfigure}
\begin{subfigure}{0.16\linewidth}
  \includegraphics[width=\linewidth]{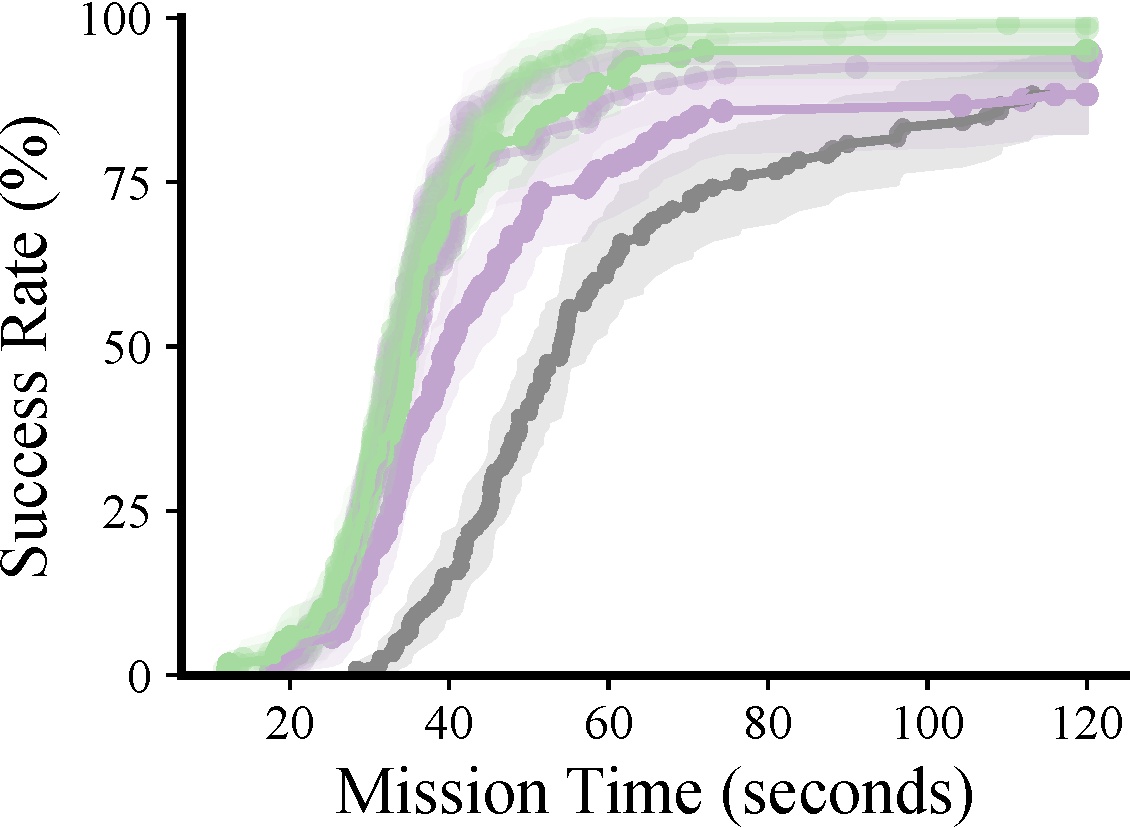}
  \caption{In-the-loop perf. sim.}
  \label{fig:in_the_loop_result}
\end{subfigure}
\begin{subfigure}{0.185\linewidth}
  \includegraphics[width=\linewidth]{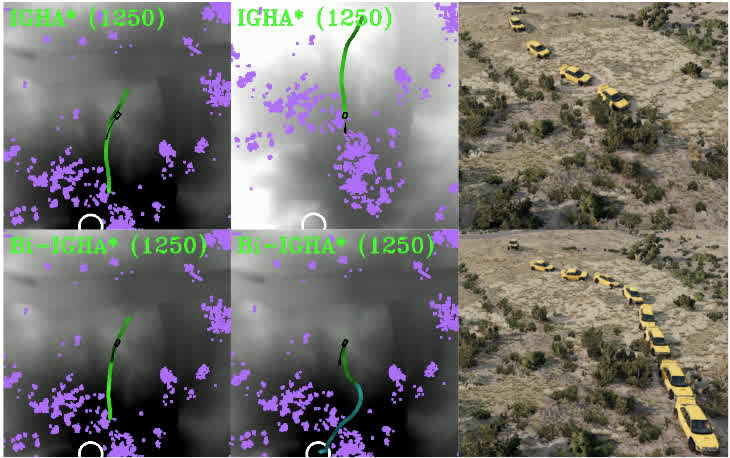}
  \caption{Starvation}
  \label{fig:qualitative_starving}
\end{subfigure}
\begin{subfigure}{0.16\linewidth}
  \includegraphics[width=\linewidth]{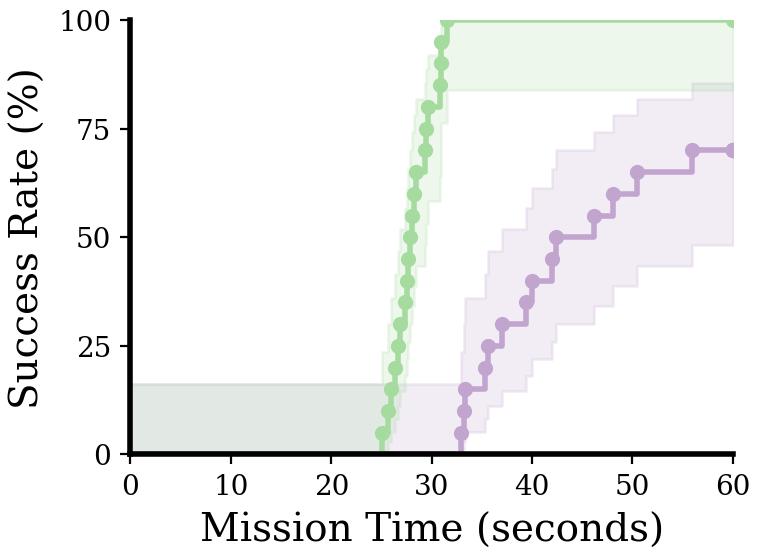}
  \caption{In-the-loop perf. real}
  \label{fig:in_the_loop_result_real}
\end{subfigure}
\begin{subfigure}{0.14\linewidth}
  \includegraphics[width=\linewidth]{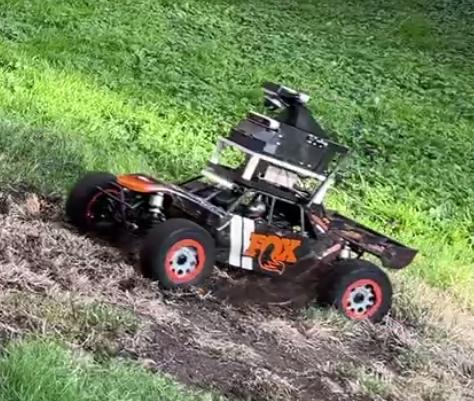}
  \caption{Real robot}
  \label{fig:real_R4}
\end{subfigure}%
\caption{
Here,
\legendsolid{0.651,0.859,0.627}: Bi-\ighastar,
\legendsolid{0.761,0.647,0.812}: \ighastar.
The playback success ~\subref{fig:playback_success}, and cost ~\subref{fig:playback_cost} metrics were used to estimate a reasonable range of expansions to use in-the-loop to exaggerate the difference between \biighastar and \ighastar.
The mission-success rate plot shows \biighastar and \ighastar
running with 5k, 2.5k, and 1.25k expansions, respectively, with darker shades implying \textbf{fewer} expansions.
We include \legendsolid{0.5,0.5,0.5}:~\sst's performance data from ~\citep{talia2025incremental} as reference~(5k exp).
Even with just 1.25k expansions, \biighastar is able to maintain mission performance on par with \ighastar using 5k expansions; while \ighastar gets stalled due to starvation
\biighastar's continues to provide high-quality solutions~(\subref{fig:qualitative_starving}).
Real-world experiments with a small-scale car-like robot~(\subref{fig:real_R4}), reaching speeds of 4.6\,m/s in an off-road environment, show the gap widening~\subref{fig:in_the_loop_result_real}: the \textit{slowest} \biighastar time is comparable to the \textit{fastest} \ighastar time.
}
\label{fig:closed_loop_exp}
\vspace{-20pt}
\end{figure*}

\section{Experiments}\label{sec: experiments}

We evaluate \biighastar in both open and closed-loop settings,
consider \ighastar as our baseline, and use 95$\%$ confidence intervals.
We count the expansions of the forward and reverse search separately.
For planning, we use dynamics models satisfying~\ref{reversibility_assumption} throughout, with particular care to avoid irreversible operations such as state clamping, constraint projection, or slip-induced information loss.
For physical dimensions, we use S.I units. 
For \nearmeet checks, we define
$\|u-v\|_2 \triangleq \|(u-v)/S\|_2$ for $u\in V_D, v\in V_{\bar{D}}$, where division by the normalizing vector $S$ is component-wise.

\subsection{Open Loop Evaluations}\label{subsec:open_loop_evals}
We present both an ablation of \biighastar with respect to LCR and a comparison against \ighastar across three systems.
First, a kinematic car~\citep{kinematic_flat} in $\mathbb{R}^3$ $(x,y,\theta)$ on urban city maps from the Moving AI benchmark~\citep{moving_AI_benchmark} (e.g., Fig.~\ref{fig:main}),
with $S=(1.0, 1.0, \pi/2)$.
Second, a kinodynamic car~\citep{non_planar_analysis} in $\mathbb{R}^4$ $(x,y,\theta,v)$, assuming no-slip and no body-frame vertical velocity, evaluated on off-road maps taken from BeamNG~\citep{beamng_tech}~(e.g.~Fig.~\ref{fig:qualitative_starving}),
with $S=(1.0, 1.0, \pi/2, 2.5)$.
Third, a hovercraft~\citep{lavalle2001randomized} in $\mathbb{R}^6$ $(x,y,\theta,v_x,v_y,\dot{\theta})$, with maps from Moving AI~\citep{moving_AI_benchmark} and BeamNG~\citep{beamng_tech}, with
$S=(1.0, 1.0, \pi/2, 2.5, 2.5, \pi/2)$.
In all environments, the vehicle dimensions (length, width) are $2.6\times1.6~\text{m}^2$.
Each environment contains 200 problems, with a 50k expansion limit. 
We evaluate two extreme hysteresis settings ($\infty, 250$; see~\citep{talia2025incremental} for details) for \ighastar~\citep{talia2025incremental} and use the same values for \biighastar (both directions).
To compare across $\mathbb{R}^3/\mathbb{R}^4/\mathbb{R}^6$, we report LCR in terms of induced volume: LCR ($R_\text{LCR}=1$), LCR/8 ($1/8^\text{th}$ volume), and LCR/64 ($1/64^\text{th}$ volume).
Note that our edges are usually longer than 1.0 meter~($\approx5$m), meaning near meets join vertices that are already closer to each other than a single edge.
Here, for simplicity, \nearmeet declares $(u,v)$ feasible if
$\|u-v\|_2<R_\text{LCR}$, where $R_\text{LCR}$ is the radius at the
corresponding resolution.

Let $n_{first}(\alg)$ be the number of expansions required by an algorithm (either \ighastar or \biighastar) to obtain a solution.
Here, \textbf{Speed-up}~$:=n_{first}(\ighastar)/n_{first}(\biighastar)$.
Let $n_{best}(\ighastar)$ be the number of expansions required by \ighastar to reach its best path cost within the expansion budget, and let $n_{best}(\biighastar)$ be the expansions required by \biighastar to reach an equivalent-or-better cost.
Here, \textbf{Speed-up} $:=n_{best}(\ighastar)/n_{best}(\biighastar)$.
We additionally record the solution level ($\text{level}_{\biighastar}, \text{level}_{\ighastar}$) for both first and best solutions.
We are interested in knowing how often bidirectional mitigation happens when there is a Speed-up.
Bidirectional mitigation manifests empirically as follows:
when a solution is found, either \forward or \backward search operates at a level~($\text{level}_{\biighastar}$) lower than \ighastar's~($\text{level}_{\ighastar}$).
We define $\mu_\leq:=\hat{p}(\text{level}_{\biighastar}\leq \text{level}_{\ighastar}\mid \text{Speed-up}>1)$ and $\mu_<:=\hat{p}(\text{level}_{\biighastar}< \text{level}_{\ighastar}\mid \text{Speed-up}>1)$.
We make the above comparisons against the best of forward and reverse \ighastar, as if an oracle informed \ighastar on whether to run in forward or reverse.
Since the first solution admits any cost, we compare \biighastar against the faster of forward/reverse to the first (any) solution.
For the best solution, we compare the expansions taken by \biighastar with those of the forward/reverse \ighastar that achieves the lowest cost.
We do this comparison to show that \biighastar is not fast just by hedging a bet in both directions.
In this section, we test the following hypotheses:
\begin{enumerate}[label=\textbf{H\arabic*}]
    \item A variant of \biighastar achieves Speedup$>1$.\label{hyp:speedup_meas}
    \item $\hat{p}(\text{Speed-up}>1)>0.5$.\label{hyp:speedup_prob}
    \item $\mu_\leq>0.5$.\label{hyp:level_prob_lax}
    \item $\mu_< >0.5$.\label{hyp:level_prob_strict}
\end{enumerate}

\textbf{Result:}
Fig.~\ref{fig:R3R4R6_speedup} shows that generally Speed-up$>1$ for the first and best path, often approaching an order of magnitude, confirming~\ref{hyp:speedup_meas}.
Fig.~\ref{tab:speedup_prob_lcr} shows that~\ref{hyp:speedup_prob} is always true for the first path, and often true for the best path.
For both of the above plots, increasing LCR increases speed up, and the speed up (or its likelihood) dips when compared to the best of forward/reverse \ighastar.
Fig.~\ref{fig:level_probability_comparison} shows that~\ref{hyp:level_prob_lax},~\ref{hyp:level_prob_strict} are always true, and provides context on the stricter condition of $\mu_<$.
Fig.~\ref{fig:Success_rate_cost_comparison_open_loop} shows a more conventional comparison and provides additional nuance to show that the results in Fig.~\ref{fig:R3R4R6_speedup},~\ref{tab:speedup_prob_lcr} are pessimistic.

\subsection{Closed Loop Evaluations}\label{subsec: closed_loop_exp}
We evaluate the planners using the $\mathbb{R}^4$ kinodynamic car model in off-road environments. 
Planning runs in closed loop with MPPI~\citep{williams2017model} (1.5\,s horizon, 1024 rollouts) in BeamNG~\citep{beamng_tech}, assuming ground-truth state and body-centric bird’s-eye-view cost and elevation maps ($100{\times}100\,\text{m}^2$, 25\,Hz).
We use the same benchmarking setup as used by the authors in ~\citep{talia2025incremental} (30 scenarios, 4 trials) and reuse its reported results for \ighastar, along with \sst~\citep{SST} as a point of reference.
Here, similar to ~\citep{ortiz2024idbrrt}, \nearmeet for $(u,v)$ additionally collision checks interpolated states between $u$ and $v$ to avoid near-meets over thin obstacles.
The vehicle dimensions are $2.6\times1.6\text{m}^2$, and
the $S=(1.0,1.0,\pi/8,1.25)\approx\text{LCR}/8$.
We use the same hysteresis value of 100 as used by \ighastar in ~\citep{talia2025incremental}, and the rate of expansions ($\approx$20k/sec) on our computer matches what was reported in~\citep{talia2025incremental}.
To bring out the difference between \ighastar and \biighastar, we operate both at 5k, 2.5k, and 1.25k expansions -- we arrive at this range of expansions by looking at the success/cost plots on playback data~(Fig.~\ref{fig:playback_success},~\ref{fig:playback_cost}).
When running the planner with fewer than 5k expansions, we slow down the rate of expansions to simulate a slower computer, such that the overall update frequency remains identical~(4 Hz, as in~\citep{talia2025incremental}).
We additionally evaluate the planner on a smaller~($0.9\times0.6\text{m}^2$) car-like off-road robot~(a modified version of the HOUND platform~\citep{talia2023demonstrating}), using $S=(0.5,0.5,\pi/8,2.0)$, in an off-road environment, with an average trial distance of 80 meters over 20 trials. 
Planning runs at a locked 4 Hz with a budget of 625 expansions in closed loop with MPPI (1.0\,s horizon, 1024 rollouts, 50 Hz) and online perception with a $30{\times}30,\text{m}^2$ map at 10 Hz. 
All modules run onboard an NVIDIA AGX Orin 64GB.
The hypothesis we want to test here is:
\begin{enumerate}[label=\textbf{H\arabic*}]
\setcounter{enumi}{4}
    \item \biighastar achieves higher success rates than \ighastar\ for a given expansion budget.
    \label{hyp:closed_loop_succ}
\end{enumerate}

\textbf{Result:}
Fig.~\ref{fig:in_the_loop_result} shows that while \ighastar's mission speed (slope of the curves) is generally close to \biighastar's, \biighastar consistently attains higher success rates under equal compute budgets, reaching comparable reliability in less time.
Fig.~\ref{fig:qualitative_starving} illustrates this difference anecdotally.
Fig.~\ref{fig:in_the_loop_result_real} shows a similar pattern to Fig.~\ref{fig:in_the_loop_result}, with a wider gap between \biighastar and \ighastar. 
This confirms~\ref{hyp:closed_loop_succ} in both simulation and real-world experiments, demonstrating the realizability of \biighastar and its benefits at lower computational budgets.

\section{Discussion}\label{sec: discuss}
This work proposes \biighastar and shows how bidirectionality overcomes a structural limitation of \ighastar, in that the improvement does not come from the reduction in search depth alone, as is usually expected with bidirectional search.
Under assumptions~\ref{heuristic_assumption}-\ref{reversibility_assumption}, \biighastar preserves \ighastar's guarantees on monotonic cost improvement and termination with finite expansions.
Our work keeps the \forward and \backward search identical in both the rate of expansions and the \shift and \activate methods used; however, the guarantees of \biighastar don't rely on this.
For \biighastar to be practically useful, a time-reversible dynamics model~(for~\ref{reversibility_assumption}) and a reliable \nearmeet implementation are required.
We leverage simplified time-reversible dynamics models as \textit{good enough} proxies for the true model and treat the \nearmeet method as a black box.
We show that useful speed-ups are achievable with a range of LCR values, show that the chosen values are reasonable by using the middle value (LCR/8) in our simulated closed-loop experiment, with a fairly simple collision-checked \nearmeet, and use the real-world experiments to show the realizability of our work.

\section*{Acknowledgements}
\footnotesize{
This work was (partially) funded by the National Science Foundation 
NRI (\#2132848) and CHS (\#2007011), DARPA RACER (\#HR0011-21-C-0171), the Office of Naval Research (\#N00014-17-1-2617-P00004 and \#2022-016-01 UW), Amazon, Technion Autonomous Systems Program (TASP), and the United States-Israel Binational Science Foundation (BSF) grant (\#2021643).
We thank NVIDIA for providing the Jetson AGX Orin used in the real-world experiments.

}
\normalsize

\footnotesize
{\bibliographystyle{IEEEtran}
\bibliography{refs}
}
\end{document}